\documentclass[lettersize,journal]{IEEEtran}

\pdfoutput=1

\usepackage{amsmath,amsfonts}
\usepackage{algorithmic}
\usepackage{algorithm}
\usepackage{array}
\usepackage[caption=false,font=normalsize,labelfont=sf,textfont=sf]{subfig}
\usepackage{textcomp}
\usepackage{stfloats}
\usepackage{url}
\usepackage{verbatim}
\usepackage{graphicx}
\usepackage{cite}
\usepackage{amssymb}
\usepackage{mathtools}
\usepackage{amsthm}
\usepackage{mathtools}
\usepackage{calc}
\usepackage{microtype}
\usepackage{graphicx}
\usepackage{booktabs} 
\usepackage{xcolor}

\newtheorem{thm}{Theorem}

\newtheorem{lem}{Lemma}

\newtheorem{cor}{Corollary}

\newtheorem{defn}{Definition}

\newtheorem{rem}{Remark}

\begin{document}

\title{FedKL: Tackling Data Heterogeneity in Federated Reinforcement Learning by Penalizing KL Divergence} 

\author{Zhijie Xie and S.H. Song,~\IEEEmembership{Senior Member,~IEEE}
\thanks{Zhijie Xie and S.H. Song are with Department of Electronic and Computer Engineering, the Hong Kong University of Science and Technology, Hong Kong. e-mail: (zhijie.xie@connect.ust.hk, eeshsong@ust.hk).}
}



\maketitle

\begin{abstract}
As a distributed learning paradigm, Federated Learning (FL) faces the communication bottleneck issue due to many rounds of model synchronization and aggregation. Heterogeneous data further deteriorates the situation by causing slow convergence. Although the impact of data heterogeneity on supervised FL has been widely studied, the related investigation for Federated Reinforcement Learning (FRL) is still in its infancy. In this paper, we first define the type and level of data heterogeneity for policy gradient based FRL systems. By inspecting the connection between the global and local objective functions, we prove that local training can benefit the global objective, if the local update is properly penalized by the total variation (TV) distance between the local and global policies. A necessary condition for the global policy to be learn-able from the local policy is also derived, which is directly related to the heterogeneity level. Based on the theoretical result, a Kullback-Leibler (KL) divergence based penalty is proposed, which, different from the conventional method that penalizes the model divergence in the parameter space, directly constrains the model outputs in the distribution space. Convergence proof of the proposed algorithm is also provided. By jointly penalizing the divergence of the local policy from the global policy with a \textit{global} penalty and constraining each iteration of the local training with a \textit{local} penalty, the proposed method achieves a better trade-off between training speed (step size) and convergence. Experiment results on two popular Reinforcement Learning (RL) experiment platforms demonstrate the advantage of the proposed algorithm over existing methods in accelerating and stabilizing the training process with heterogeneous data.
\end{abstract}


\section{Introduction\label{sec:Introduction}}

\IEEEPARstart{F}{ederated} Learning (FL) was proposed as a distributed
Machine Learning (ML) scheme where data from distributed devices are utilized for training without violating the user privacy. The idea has been successfully implemented (e.g. the Google Keyboard) \cite{DBLP:journals/corr/abs-1812-03239}
and attracted much research attention \cite{10.1145/3298981,9084352}.
One of the key challenges for FL is the communication bottleneck due to the frequent communication of big ML models. There are mainly two methods to mitigate the issue, i.e., reducing the communication workload per round and decreasing the number of communication rounds. To this end, model compression has been proposed to reduce the communication workload per round \cite{DBLP:journals/corr/KonecnyMYRSB16,10.5555/3294771.3294934}, and another line of research tried to improve the convergence speed to minimize the number of communication rounds \cite{10.5555/3454287.3455281,JMLR:v22:20-147,Yu_Yang_Zhu_2019}.
Unfortunately, data heterogeneity between difference devices significantly slows the convergence down \cite{DBLP:journals/corr/abs-2106-06843,li2020on}.

As an important ML paradigm, Reinforcement Learning
(RL) trains the policy of an agent in an iterative way by interacting with the environment. In each time step, the agent makes a decision to manipulate the environment and receives a reward for its action, where the objective of the training is to maximize the long term reward\cite{Sutton1998}. Due to the distributed data and computing power in large-scale applications such as autonomous driving, the training of RL algorithms under the FL framework is inevitable. Unfortunately, many challenges faced by supervised FL, e.g., the data heterogeneity and communication bottleneck, are still valid and even worse for FRL \cite{DBLP:journals/corr/abs-2108-11887}. For example, centralized policy gradient methods already suffer from high variance which is detrimental to the convergence speed and training performance \cite{10.5555/3044805.3044850,JMLR:v21:18-012,NIPS1999_464d828b}, and data heterogeneity imposes another layer of difficulty for the convergence of FRL. 

Data heterogeneity in FL refers to the situation where the data collected by different devices are not independent and identically distributed (IID) \cite{9084352,DBLP:journals/corr/abs-2102-02079}. Despite its empirical success, FedAvg does not perform well on highly skewed non-IID data \cite{li2020on,DBLP:journals/corr/abs-1806-00582,pmlr-v54-mcmahan17a}. As a result, many innovative methods have been proposed to tackle the data heterogeneity issue \cite{DBLP:journals/corr/abs-2106-06843}. Data-level approaches aim to modify the data distribution to address the root of the problem. For example, a data-sharing strategy \cite{DBLP:journals/corr/abs-1806-00582}
was proposed to create a small but globally shared dataset to mitigate the impact of heterogeneous data. On the other hand, algorithm-level approaches try to solve the problem by improving the local training algorithms. Along this line of research, FedProx \cite{MLSYS2020_38af8613} added a weighted proximal term to the local objective function to constrain the divergence of the local model from the global model. Finally, system-level approaches focus on client clustering/selection \cite{10.1109/INFOCOM41043.2020.9155494} to counterbalance the bias introduced by non-IID data and speed up convergence. 

The heterogeneity issue of FRL is different from that of the supervised FL and the related study is very limited
\cite{8772088,DBLP:journals/corr/abs-1910-06001,LimHK2020}. There are mainly two types of RL methods, namely, the action-value method and the policy gradient (PG) method \cite{DBLP:journals/corr/abs-2108-11887,Sutton1998}. The first performs action selection according to the value function, while the second utilizes ML models to approximate the policy and update the policy by gradient based algorithms. Data heterogeneity affects different RL methods in different ways, and thus requires different solutions. For the action-value method, Lifelong Reinforcement Learning (LFRL) \cite{8967908} explored the combination of Q-learning and FL in real-world applications, where the knowledge learned in one particular environment is transferred to another. Federated RL (FedRL) \cite{DBLP:journals/corr/abs-1901-08277} considered the scenario where some agents can only observe part of the environment and have no access to the reward. To handle this issue, agents maintain their individual Q-networks in addition to the global Q-network, and exchange their predicted Q value for cooperation. For policy gradient methods, 
Federated Transfer Reinforcement Learning (FTRL) \cite{DBLP:journals/corr/abs-1910-06001} proposed an online transfer process to numerically align the scale of the observations and actions in different environments so that the policy learnt in one environment is portable to another. Given that devices of the same type may have different dynamics, a FRL architecture, which adopts Proximal Policy Gradient (PPO) and Multi-Agent Deep Deterministic Policy Gradient (MADDPG) \cite{DBLP:journals/corr/SchulmanWDRK17,NIPS1999_6449f44a,10.5555/3295222.3295385} to FRL in Internet of Things (IoT) networks, was proposed in \cite{LimHK2020,Lim2021FederatedRL}. This FRL scheme alternates between sharing gradient and transferring mature models to other agents. For offline learning, data heterogeneity exists in different
local datasets. To address this problem, Federated Dynamic Treatment Regime (FDTR) \cite{https://doi.org/10.48550/arxiv.2206.05581} proposed the first federated policy optimization algorithm for offline RL. There are also works that aim to improve the communication efficiency with heterogeneous data \cite{9500603,DBLP:journals/corr/abs-2103-13026}. For instance, Federated Multi-Agent Reinforcement Learning (FMARL) \cite{DBLP:journals/corr/abs-2103-13026} proposed
to gradually decay the gradients during local training to avoid model divergence and improve communication efficiency. QAvg and PAvg \cite{pmlr-v151-jin22a} theoretically proved that federated Q-Learning and federated policy gradient can converge in tabular case when agents' environments are heterogeneous.

In this paper, we will investigate the data heterogeneity issue in FRL systems whose agents employ PG methods with an advantage estimator \cite{NIPS1999_464d828b,DBLP:journals/corr/SchulmanMLJA15}. For that purpose, we first classify the types of data heterogeneity in FRL and define the heterogeneity level. By analyzing the connection between the global and local objective functions, we prove that although it is possible to improve the global objective by updating the local policy according to the local environment, the local update must be properly constrained. Based on the theory, we propose an algorithm that utilizes a Kullback–Leibler (KL) divergence based term to penalize the local training for diverging too far away from the global policy and name it the FedKL algorithm. The proposed KL penalty will be referred to as the global penalty because it guarantees that the local training of one agent will not diverge too far away from the global policy. Note that this global penalty is different from the per-iteration penalty utilized in centralized RL \cite{DBLP:journals/corr/SchulmanWDRK17}. We will investigate the benefits of adding the global penalty and how its interplay with the local per-iteration penalty helps mitigate the data heterogeneity issue in FRL. Experiment results show that FedKL is able to stabilize training and speed up convergence when compared with existing algorithms, including FedAvg, FedProx and FMARL.

The major contributions of this work can be summarized as follows: 
\begin{itemize}
\item We classify the data heterogeneity of FRL into two types, due to different initial state distributions and different environment dynamics, respectively. Then, we define the level of heterogeneity, based on which a necessary condition for the local policy improvement to be beneficial for the global objective is derived. 

\item We prove that the improvement of the local policy on one agent can benefit the global objective, as long as the local policy update is properly penalized by the total variation (TV) distance between the local and global policies.  Based on the theoretical result, we propose the FedKL algorithm and adopt several approximations to simplify the implementation of FedKL. Convergence proof of the proposed algorithm is also provided.

\item Experiment results show that FedKL outperforms existing algorithm-level solutions with faster convergence and better stabilization. Several interesting insights are also obtained: 1) The proposed heterogeneity level is very effective and can be utilized to predict the learning behavior; 2) The KL divergence based penalty is more effective than the penalty defined in the parameter space; 3) The proposed global penalty enables a larger learning step size, which speeds up the convergence without causing divergence.  
\end{itemize}

\section{Preliminaries}

\subsection{Federated Learning}
A typical FL system consists of one central server and $N$ agents (clients/devices), where the data of the agents are jointly utilized to train a global model without violating their privacy. In particular, the global optimization problem can be formulated as \cite{pmlr-v54-mcmahan17a}
\begin{alignat}{1}
\min_{\theta}f(\theta) & =\sum_{n=1}^{N} q_{n} f_{n}(\theta),
\end{alignat}
where $f_{n}(\theta)$ represents the local objective function, e.g. the loss function, of the $n$-th agent with respect to the model parameters $\theta$. $q_{n}$ is a weighting factor often set to $q_{n}=l_{n}/L$, with $l_n$ denoting the number of data points at the $n$-th agent and $L=\sum_{n=1}^N l_{n}$.  In each communication round, $K$ out of the total $N$ agents are selected and the central server broadcasts the global model to the selected agents. Then, each agent updates its local model utilizing the local data and uploads the trained model to the server for model aggregation. This process will repeat for some rounds until a preset learning objective is achieved. 

\subsection{Federated Reinforcement Learning}
The system setup of FRL is similar, i.e., one central server
federates the learning of $N$ distributed agents. In the $t$-th training round, the central server broadcasts the current global policy $\pi^{t}$
to $K$ selected agents which will perform $I$ iterations of local
training. In each iteration, the agent interacts with its environment
to collect $T$ time-steps of data, and optimizes its local policy according
to the local objective with gradient ascent methods. At the end of each round, the training results will be uploaded to the server for aggregation. An example of FRL systems for autonomous driving is illustrated in Fig. \ref{fig:system_model}.

\begin{figure*}[!t]
\centering
\includegraphics[width=1.5 \columnwidth]{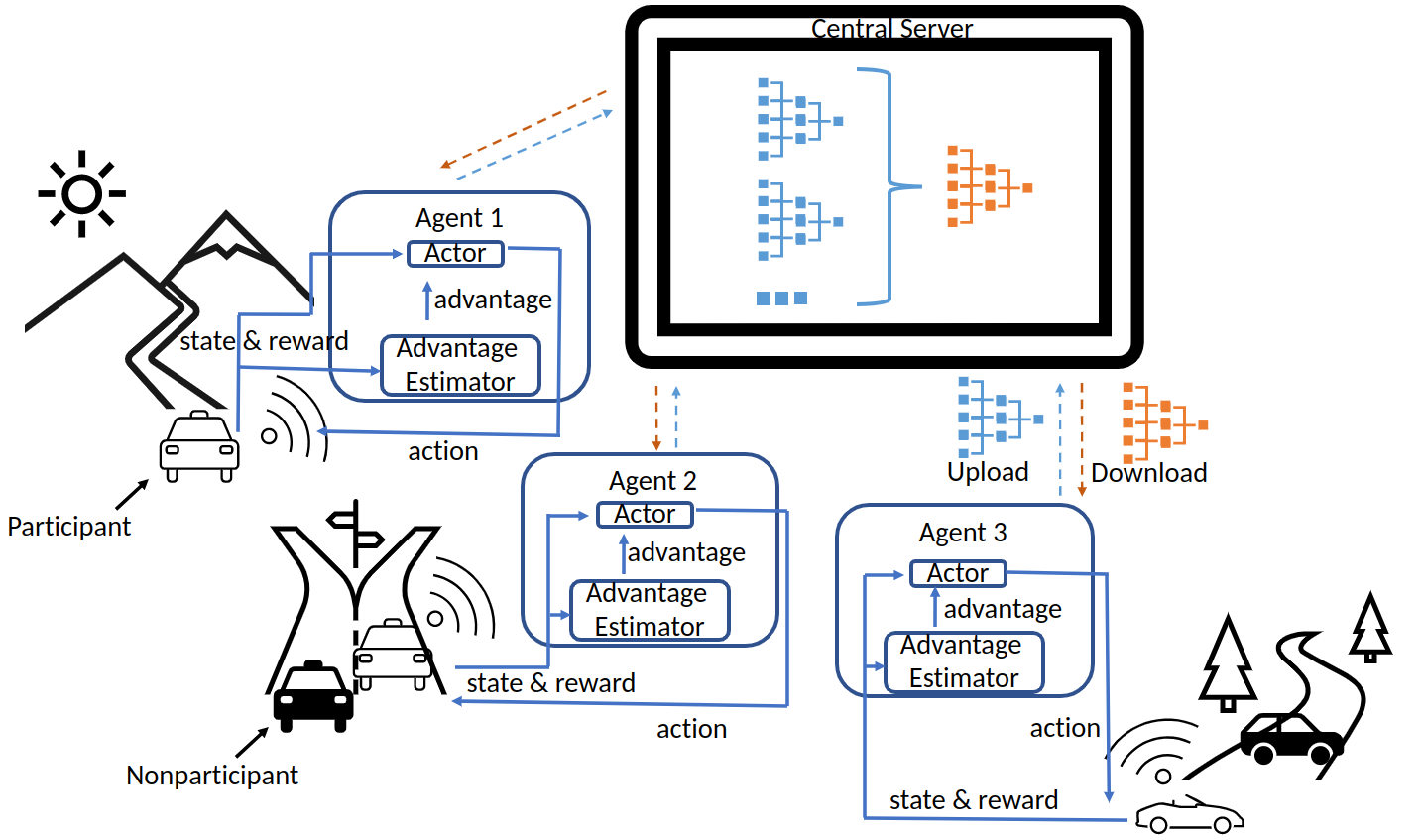} \caption{An example of FRL systems: autonomous driving. Many vehicles run on different roads. In each training round, selected vehicles (white cars) will participate in the training. With the PG method, policies are represented by ML models, which will be communicated between agents and the server for model aggregation and synchronization. The model aggregation is the same as that of FedAvg.}
\label{fig:system_model}
\end{figure*}

We model the local training problem of each agent as a finite Markov Decision Process (MDP). Accordingly, the FRL system consists of $N$ finite MDPs $\{ (\mathcal{S},\mu_{n},\mathcal{A},P_{n},\gamma) \vert n \in[1,N] \} $,
where $\mathcal{S}$ denotes a finite set of states, $\mu_{n}$ represents the initial state distribution of the n-th MDP, $\mathcal{A}$ is a finite set of actions and $\gamma\in(0,1)$ is the discount factor. The dynamics function $P_{n}(s^{\prime},r\vert s,a):\mathcal{S}\times\mathbb{R}\times\mathcal{S}\times\mathcal{A}\to\mathbb{R}$
represents the probability that the n-th MDP transits from state $s$ to $s^{\prime}$ and obtain a reward of $r\in\mathbb{R}$
after taking action $a$, which characterizes the environment dynamics of the n-th MDP \cite{Sutton1998}. We can define the transition probability $P_{n}(s^{\prime}\vert s,a):\mathcal{S}\times\mathcal{S}\times\mathcal{A}\to\mathbb{R}$ as $P_{n}(s^{\prime}\vert s,a)=\sum_{r\in\mathbb{R}}P_{n}(s^{\prime},r\vert s,a)$. Furthermore, based on the dynamics function, we can define the reward function $\mathcal{R}_{n}(s,a):\mathcal{S}\times\mathcal{A}\to\mathbb{R}$ as $\mathcal{R}_{n}(s,a)=\sum_{r\in\mathbb{R}}r\sum_{s^{\prime}\in \mathcal{S}}P_{n}(s^{\prime},r\vert s,a)$,
which gives the expected reward of the n-th MDP for taking action $a$ in state $s$. As a result, the $n$-th MDP $\mathcal{M}_n$ can be represented by a 5-tuple $(\mathcal{S},\mu_{n},\mathcal{A},P_{n},\gamma)$
sharing the same state and action space with other agents, but with possibly different initial state distributions and environment dynamics. An agent's behavior is controlled by a stochastic policy $\pi_{n}:\mathcal{S}\times\mathcal{A}\rightarrow[0,1]$
which outputs the probability of taking an action $a$ in a given state
$s$. Specifically, each agent $n$ aims to train a local policy to maximize its expected discounted reward over the long run
\begin{alignat}{1}
\eta_{n}(\pi) = \mathbb{E}_{s_{0} \thicksim \mu_{n},a_{t} \thicksim \pi,s_{t+1} \thicksim P_{n}} \left[\sum_{t=0}^{\infty}\gamma^{t}\mathcal{R}_{n}(s_{t},a_{t})\right],\label{eq:1}
\end{alignat}
where the notation $\mathbb{E}_{s_{0} \thicksim \mu_{n},a_{t} \thicksim \pi,s_{t+1} \thicksim P_{n}}$ indicates that the reward is averaged over all states and actions according to the initial state distribution, the transition probability and the policy. Accordingly, the optimization problem for FRL can be formulated as 
\begin{alignat}{1}
\max_{\pi}\eta(\pi) & =\sum_{n=1}^{N}q_{n}\eta_{n}(\pi). \label{eq:2}
\end{alignat}


Note that $\mu_{n}$ is added to the formulation in (\ref{eq:1}) to represent the different initial state distributions for different agents and will be discussed in details later. The above formulation covers both IID and non-IID cases. In particular, the different MDPs, i.e., different initial state distributions and environment dynamics, represent the heterogeneous environments experienced by different agents (non-IID). All MDPs will be identical for the IID case \cite{fan2021fault}.  

\subsection{Policy Gradient Methods}
In this section, we will introduce the policy gradient (PG) method \cite{NIPS1999_464d828b}, which is taken by the agents to update their local policies in this paper. While action-value methods learn the values of actions and then select actions based on their estimated values, PG methods directly learn the decision-making policy.

To update the policy by gradient ascent methods, we need to estimate the policy gradient. For that purpose, we first define the state-value function and the action-value
function. Specifically, the state-value function
$V_{\pi, P_{n}}(s)$ of a state $s$ is the expected return when the agent
starts from state $s$ and follows policy $\pi$ thereafter:
\begin{alignat}{1}
V_{\pi,P_{n}}(s) & =\mathbb{E}_{\pi,P_{n}}\left[\sum_{l=0}^{\infty} \gamma^{l}\mathcal{R}(s_{t+l},a_{t+l}) \vert s_{t}=s\right],
\end{alignat}
where the expectation is performed over actions sampled from policy $\pi$ and the next state is sampled based on the transition probability $P_{n}$. Similarly, the action-value function $Q_{\pi,P_{n}}(s,a)$ of a state-action pair is the expected return when the agent starts from state $s$, takes the action $a$, and follows policy $\pi$ thereafter:
\begin{alignat}{1}
Q_{\pi,P_{n}}(s,a) & = \mathbb{E}_{\pi,P_{n}} \left[\sum_{l=0}^{\infty} \gamma^{l}\mathcal{R}(s_{t+l},a_{t+l}) \vert s_{t}=s,a_{t}=a\right].
\end{alignat}
By subtracting the state-value function from the action-value function, we can define the advantage function $A_{\pi,P_{n}}(s,a) = Q_{\pi,P_{n}}(s,a)-V_{\pi,P_{n}}(s),$ which measures the value of action $a$ in state $s$. The author of \cite{NIPS1999_464d828b} proved that the parameters $\theta$ of a differentiable policy $\pi(\theta)$ can be updated using the policy gradient 
\begin{eqnarray}
\hat{g} = \mathbb{E}_{\pi,P_{n}}\left[\nabla_{\theta} \log\pi(a\vert s)A_{\pi,P_{n}}(s,a)\right],
\end{eqnarray}
where $\nabla_{\theta}$ denotes the derivative with respect to $\theta$.

However, the high variance of the gradient estimation often leads
to large policy updates and causes instability. As a result, the step size must be limited. For that purpose, \cite{DBLP:journals/corr/SchulmanLMJA15,10.5555/645531.656005}
investigated the performance gap for taking two different polices $\pi$ and $\pi^{\prime}$,
\begin{alignat}{1}
& \eta_{n}(\pi^{\prime})-\eta_{n}(\pi)\nonumber\\
= & \mathbb{E}_{s\thicksim\rho_{\pi^{\prime},\mu_{n},P_{n}}}\left[ \mathbb{E}_{a\thicksim\pi^{\prime}(a|s)}\left[A_{\pi, P_{n}}(s,a)\right] \right], \label{eq:diffeta}
\end{alignat}
where 
\begin{alignat}{1}
\rho_{\pi,\mu_{n},P_{n}}(s) & = \sum_{t=0}^{\infty}\gamma^{t} \mbox{Pr}(s_{t}=s|\pi,\mu_{n},P_{n}),\label{eq:rho}
\end{alignat}
denotes the unnormalized discounted visitation frequency. 

Equation (\ref{eq:diffeta}) implies that the reward can be increased by any policy update that improves the expected advantage. However, it is difficult to optimize the
right-hand side (RHS) of (\ref{eq:diffeta}) since $\rho_{\pi^{\prime},\mu_{n},P_{n}}$
is dependent on $\pi^{\prime}$. Therefore, 
the following surrogate objective was proposed
\begin{alignat}{1}
L_{\pi,\mu_{n},P_{n}}(\pi^{\prime}) & =\eta_{n}(\pi)+\mathbb{A}_{\pi,\mu_{n},P_{n}}(\pi^{\prime}),\label{eq:surrogate}
\end{alignat}
where
\begin{alignat}{1}
\mathbb{A}_{\pi,\mu_{n},P_{n}}(\pi^{\prime}) = \mathbb{E}_{s\thicksim\rho_{\pi,\mu_{n},P_{n}}} \left[ \mathbb{E}_{a\thicksim\pi^{\prime}(\cdot|s)}\left[A_{\pi,P_{n}}(s,a)\right]\right],\label{eq:policyadvantage}
\end{alignat}
denotes the policy advantage that measures the advantage of policy $\pi^{\prime}$ over $\pi$, starting from state $s\thicksim\mu_{n}$ and following $\pi^{\prime}$. $L_{\pi,\mu_{n},P_{n}}$ is an approximation
to $\eta_{n}(\pi^{\prime})$ but easier to optimize because it replaces the dependency
of $\rho_{\pi^{\prime},\mu_{n},P_{n}}$ on $\pi^{\prime}$ with that
of $\rho_{\pi,\mu_{n},P_{n}}$ on $\pi$.

Although the above step involves an approximation, it is shown that a sufficiently small step that improves $L_{\pi,\mu_{n},P_{n}}(\pi^{\prime})$ will also improve $\eta_{n}(\pi^{\prime})$. 
Specifically, 
the following policy improvement bound was derived for any policies $\pi$ and $\pi^{\prime}$
\begin{eqnarray}
\label{TRPOObj}
\eta_{n}(\pi^{\prime}) & \ge L_{\pi,\mu_{n},P_{n}}(\pi^{\prime})-\text{\ensuremath{c  D_{KL}^{max}(\pi,\pi^{\prime}),}}\label{eq:trpo1}
\end{eqnarray}
or equivalently
\begin{eqnarray}
\label{TRPOreturndiff}
\eta_{n}(\pi^{\prime}) - \eta_{n}(\pi) & \ge \mathbb{A}_{\pi,\mu_{n},P_{n}}(\pi^{\prime})-\text{\ensuremath{c  D_{KL}^{max}(\pi,\pi^{\prime}),}}\label{eq:trpo2}
\end{eqnarray}
where $c=\frac{4\max_{s,a}|A_{\pi,P_{k}}(s,a)|\gamma}{(1-\gamma)^{2}}$ is a constant and 
$D_{KL}^{max}(\pi,\pi^{\prime})  =\max_{s} D_{KL} \left( \pi(\cdot\vert s) \Vert \pi^{\prime}(\cdot\vert s) \right)$ denotes the maximum KL divergence between two policies $\pi$ and $\pi^{\prime}$
among all states. As a result, one can optimize Equation (\ref{eq:1}) by optimizing
the surrogate objective. A tighter bound for the policy improvement was later derived \cite{pmlr-v70-achiam17a}
\begin{alignat}{1}
\label{eq:cporeturndiff}
\eta_{n}(\pi^{\prime}) - \eta_{n}(\pi) & \ge \mathbb{A}_{\pi,\mu_{n},P_{n}}(\pi^{\prime})\\
& \quad-\text{\ensuremath{c^{\text{CPO}}  \mathbb{E}_{s\thicksim\rho_{\pi,\mu_{n},P_{n}}} \left[D_{TV}(\pi(\cdot\vert s)\Vert\pi^{\prime}(\cdot\vert s))\right],}}\nonumber
\end{alignat}
where $c^{\text{CPO}}=\frac{2\max_{s}\left|\mathbb{E}_{a\thicksim\pi^{\prime}(\cdot|s)}\left[A_{\pi,P_{k}}(s,a)\right]\right|\gamma}{(1-\gamma)^{2}}$ is upper bounded by $c$ and $D_{TV}(\pi(\cdot\vert s)\Vert\pi^{\prime}(\cdot\vert s))  =\frac{1}{2}\sum_{a}\left|\pi(a\vert s)-\pi^{\prime}(a\vert s)\right|$
denotes the TV distance between two discrete probability
distributions $\pi(\cdot\vert s)$ and $\pi^{\prime}(\cdot\vert s)$.
For the continuous case, it is straightforward to replace the summation by
integration. In this work, (\ref{eq:cporeturndiff}) is useful to derive the policy improvement bound of FRL.



\section{Data Heterogeneity and its Impact on FRL}
In this section, we will first investigate the types of data heterogeneity in FRL and define the heterogeneity level. Then, by studying the connection between the local and global objectives, we will investigate how the local update can help the global objective and derive a necessary condition for the global policy to be learn-able.

\subsection{Types of Data Heterogeneity in FRL}
Given (\ref{eq:rho}), the expected discounted reward for the $n$-th MDP $\mathcal{M}_{n}$, defined in (\ref{eq:1}), can be written as
\begin{alignat}{1}
\eta_{n}(\pi) =\sum_{s}\rho_{\pi,\mu_{n},P_{n}}(s)\sum_{a}\pi(a|s)\mathcal{R}_{n}(s,a).\label{expectationeta}
\end{alignat}
It can be observed that the performance of a policy on $\mathcal{M}_{n}$ is jointly determined by its initial state distribution $\mu_{n}$ and the dynamics function $P_{n}$, both of which can be either stochastic or deterministic. As a result, agents with different initial state distributions and environment dynamics will end up with different policies, causing the heterogeneity issue. Therefore, we classify the data heterogeneity in FRL into two types:
\begin{enumerate}
\item Heterogeneity due to initial state distribution;
\item Heterogeneity due to environment dynamics.
\end{enumerate}
For ease of illustration, we consider two extreme cases where only
one type of heterogeneity exists. Consider two finite MDPs, $\mathcal{M}_{1}$
and $\mathcal{M}_{2}$, with the same environment dynamics, $P_{1}=P_{2}$
but different initial state distributions, $\mu_{1} \neq \mu_{2}$. This
corresponds to the situation where two self-driving vehicles have the
same car model and training routine, but tend to enter the same road from different entrances. 
Now, consider the other extreme case, where $\mu_{1} = \mu_{2}$ but $P_{1} \neq P_{2}$, i.e., the two MDPs have the same initial state distribution but different environment dynamics. This can be utilized to model the situation where two different cars always enter the same road at the same entrance. 
The two types of heterogeneity can occur simultaneously. But, the heterogeneity due to different initial state distributions is of more concern for episodic tasks whose on-policy distribution is different from the stationary distribution \cite{Sutton1998} and heavily relies on the initial state distribution. 

\subsection{Level of Heterogeneity}
In the following, we define a measurement regarding the level (impact) of
heterogeneity which will be utilized to investigate the connection 
between the global and local policy advantages. For ease of illustration, we define $\mathbf{D}_{\pi, \mu_{n}, P_{n}}$ as a $\left|\mathcal{S}\right|\times\left|\mathcal{S}\right|$
diagonal matrix with $\rho_{\pi,\mu_{n},P_{n}}$ being the $i$-th diagonal entry, where $\left|\mathcal{S}\right|$ denotes the cardinality of $\mathcal{S}$. 
Similarly, we denote $\mathbf{\Pi}_{\pi}$ as a $\left|\mathcal{S}\right|\times\left|\mathcal{A}\right|$
matrix whose $(i,j)$th entry is $\pi(a_{j}\vert s_{i})$, 
and $\mathbf{A}_{\pi,P_{n}}$
as a $\left|\mathcal{S}\right|\times\left|\mathcal{A}\right|$ matrix
whose $(i,j)$th entry is $A_{\pi,P_{n}}(s_{i},a_{j})$. 
With the above notation, we can rewrite the policy advantage in (\ref{eq:policyadvantage}) as 
\begin{alignat}{1}
\mathbb{A}_{\pi,\mu_{n},P_{n}}(\pi^{\prime}) & =\sum_{s}\rho_{\pi,\mu_{n},P_{n}}(s)\sum_{a}\pi^{\prime}(a|s)A_{\pi,P_{n}}(s,a)\\
 & =\mbox{Tr} \left(\mathbf{D}_{\pi,\mu_{n},P_{n}}\mathbf{A}_{\pi,P_{n}}\mathbf{\Pi}_{\pi^{\prime}}^{T} \right),
 \label{eq:vectorpolicyadvantage}
\end{alignat}
where ${\rm Tr}(\cdot)$ denotes the trace operator. It can be observed that the advantage of taking a new policy $\pi'$ on a specific agent $n$ is realized through the matrix product $\mathbf{D}_{\pi,\mu_{n}, P_{n}} \mathbf{A}_{\pi,P_{n}}$.  Thus, we propose the following definition regarding the level of heterogeneity.
Without loss of generality, we assume that all the states
are reachable, i.e., $\rho_{\pi,\mu_{n},P_{n}}(s) > 0$ for all $s$ and $\pi$.
\begin{defn}
\label{def:impacthete}(Level of heterogeneity). We define the
level of heterogeneity of the n-th agent under policy $\pi$ by 
the Frobenius norm of the matrix 
\begin{eqnarray}
\label{B}
\mathbf{B}_{\pi,\mu_{n},P_{n}}=\sum_{k=1}^{N} q_{k} \mathbf{D}_{\pi,\mu_{n},P_{n}}^{-1} \mathbf{D}_{\pi,\mu_{k},P_{k}} \mathbf{A}_{\pi,P_{k}} - \mathbf{A}_{\pi,P_{n}}
\end{eqnarray}
as $\left\Vert \mathbf{B}_{\pi,\mu_{n},P_{n}}\right\Vert_{F}$. 
\end{defn}
Note that $\mathbf{B}_{\pi,\mu_{n},P_{n}}$ measures the deviation of the policy advantage of the $n$-th agent from the average policy advantage of all agents. Given the local policy advantage 
directly affects the policy gradient, its deviation from the average will lead to potential gradient divergence. For the IID case, $\mathbf{B}_{\pi,\mu_{n},P_{n}}=\mathbf{0}$ and thus $\left\Vert \mathbf{B}_{\pi,\mu_{n},P_{n}}\right\Vert_{F}=0$ for arbitrary policy $\pi$. With non-IID data, $\left\Vert \mathbf{B}_{\pi,\mu_{n},P_{n}}\right\Vert_{F}$ will be positive and its magnitude indicates the level of heterogeneity of the $n$-th agent. 
Intuitively, if the level of heterogeneity of one agent is too high, i.e., the policy advantage calculated over its MDP is too different from that of the others, improving the global policy according to the local environment may not be beneficial for the whole system. In Corollary \ref{cor}, we will prove a necessary condition, with respect to $\left\Vert \mathbf{B}_{\pi,\mu_{n},P_{n}}\right\Vert_{F}$, for the global policy to be learn-able from the local policy. 

\subsection{Connection between the Global and Local Objectives in FRL}
In FRL, the local objective $\eta_{n}(\pi)$ serves as a proxy of the global objective $\eta(\pi)$. Specifically, each agent tries to optimize its local objective, but the global objective is the actual performance measure. Given the global policy $\pi^{t}$ in the $t$-th training round, the $n$-th agent interacts with its environment and obtains a new local policy $\pi_{n}^{t+1}$. Without access to any information of other agents, the $n$-th agent can only update the local policy by optimizing the local objective $\eta_{n}(\pi)$. However, a good local performance does not necessarily lead to a good global performance, especially with data heterogeneity. The question we want to answer is how we can optimize the local policy to best improve the global performance. For that purpose, we start from the local improvement.

It follows from (\ref{eq:cporeturndiff}) that
\begin{alignat}{1}
& \eta_{n}(\pi_{n}^{t+1}) - \eta_{n}(\pi^{t})\nonumber\\
\ge & \mathbb{A}_{\pi^{t},\mu_{n},P_{n}}(\pi_{n}^{t+1})\nonumber\\
& -\text{\ensuremath{c^{\text{CPO}}  \mathbb{E}_{s\thicksim\rho_{\pi^{t},\mu_{n},P_{n}}} \left[D_{TV}(\pi^{t}(\cdot\vert s)\Vert\pi_{n}^{t+1}(\cdot\vert s))\right],}}
\label{eq:cpo}
\end{alignat}
which indicates that the improvement brought by the local policy
$\pi_{n}^{t+1}$ to the $n$-th agent over the current global policy $\pi^{t}$ is lower bounded by the policy advantage. The next question is whether we can bound the global improvement similarly. For that purpose, we first investigate the relation between the global and local policy advantage, and the result is given in the following theorem. 
\begin{thm}
\label{thm:1} The following bound holds for all agents $n=1,...,N$
\begin{alignat}{1}
\label{eq:thm1}
& \sum_{k=1}^{N}q_{k}\mathbb{A}_{\pi^{t},\mu_{k},P_{k}}(\pi_{n}^{t+1})\\
\ge & \mathbb{A}_{\pi^{t},\mu_{n},P_{n}}(\pi_{n}^{t+1}) - \alpha \mathbb{E}_{s\thicksim\rho_{\pi^{t},\mu_{n},P_{n}}} \left[D_{TV}(\pi(\cdot\vert s)\Vert\pi^{\prime}(\cdot\vert s))\right],\nonumber
\end{alignat}
where $\alpha=2\left\Vert \mathbf{B}_{\pi,\mu_{n},P_{n}}\right\Vert_{F}$.
\end{thm}
The proof is given in Appendix \ref{proof:theorem1}. Note that the left-hand-side (LHS) of (\ref{eq:thm1}) represents the average policy advantages of $\pi_{n}^{t+1}$ over $\pi^{t}$ when applying $\pi_{n}^{t+1}$ on all agents and the RHS represents the penalized policy advantage at the n-th agent. Thus, Theorem \ref{thm:1} indicates that the policy advantage of $\pi_{n}^{t+1}$ on all agents is lower bounded by its local policy advantage with a penalty term. As a result,  improving the penalized local policy advantage of a given agent may improve its local policy's performance on all agents.

Based on Theorem I, we have the following corollary. 
\begin{cor}
\label{cor}
The condition 
\begin{eqnarray}
\label{BleD1}
\left\Vert \mathbf{B}_{\pi,\mu_{n},P_{n}}\right\Vert _{F}<\left\Vert \mathbf{A}_{\pi,P_{n}}\right\Vert_{F}
\end{eqnarray}
is a necessary condition for the local policy update to be able to improve the global objective.
\end{cor}
\noindent The proof is given in Appendix \ref{corproof}. 

\begin{rem}
\label{rem:NecessaryCondition1} 
The RHS of (\ref{BleD1}) measures the policy advantage and can be regarded as the potential of the local policy to be further improved according to the local environment. The LHS of (\ref{BleD1}) indicates the heterogeneity level of the $n$-th agent. Corollary \ref{cor} implies that, in order to improve the global objective by optimizing the local objective, we need $\left\Vert \mathbf{A}_{\pi,P_{n}}\right\Vert_{F} - \left\Vert \mathbf{B}_{\pi,\mu_{n},P_{n}}\right\Vert_{F} > 0$. So, the heterogeneity level serves as a penalty, i.e., the local policy's contribution to the global objective is penalized by the heterogeneity level. In other words, 
even if the local training is beneficial for the local objective, it may not help the global objective unless the improvement is larger than the heterogeneity level. We have more to say about this in the experiment result section.   
\end{rem}

\begin{rem}
\label{rem:NecessaryCondition} 
Note that Corollary \ref{cor} includes the IID case as a special case. In particular, without data heterogeneity, the LHS of (\ref{BleD1}) will be zero. Thus, the necessary condition holds as long as the RHS is greater than zero, which indicates that it is much easier for the local training to help the global training. This makes sense because for the IID case, all agents have the same MDP, thus any improvement for one agent will probably be beneficial for all agents. 
\end{rem}

However, the average policy advantage is not a direct measure of the final performance, as it is a proxy for the improvement of the global objective. A more useful result should directly measure
the improvement in terms of the global expected discounted reward. For that purpose, we have the following corollary. 
\begin{cor}
\label{cor:1} Based on Theorem \ref{thm:1}, we have the following
inequality showing the relation between the global
objective and the local update:
\begin{alignat}{1}
\label{eq:cor1}
\eta(\pi_{n}^{t+1}) & \ge \sum_{k=1}^{N}q_{k}\eta_{k}(\pi^{t})+\mathbb{A}_{\pi^{t},\mu_{n},P_{n}}(\pi_{n}^{t+1})\nonumber\\
& \quad-(\alpha+\beta)\mathbb{E}_{s\thicksim\rho_{\pi^{t},\mu_{n},P_{n}}} \left[D_{TV}(\pi^{t}(\cdot\vert s)\Vert\pi_{n}^{t+1}(\cdot\vert s))\right]\nonumber\\
& \quad-\delta D_{TV}^{\max}(\pi^{t},\pi_{n}^{t+1}),
\end{alignat}
where $\delta=2 \sum_{k=1}^{N}q_{k} \frac{4\epsilon_{k}\gamma}{(1-\gamma)^{2}} D_{TV}(\rho_{\pi^{t},\mu_{k},P_{k}} \Vert \rho_{\pi^{t},\mu_{n},P_{n}})$, $\beta=\sum_{k=1}^{N} q_{k} \frac{4\epsilon_{k}\gamma}{(1-\gamma)^{2}}$
and $\epsilon_{k}=\max_{s,a}\left|A_{\pi^{t},P_{k}}(s,a)\right|$ denotes the maximum action advantage among all state and action pairs.
\end{cor}
The proof is given in Appendix \ref{proof:corollary1}. The LHS of (\ref{eq:cor1}) represents the average of expected discounted reward of the n-th local policy $\pi_{n}^{t+1}$ over all agents. The first term on the RHS of (\ref{eq:cor1}) is a constant, representing the return of the current policy, and the second term represents the policy advantage of $\pi_{n}^{t+1}$ over $\pi_{n}^{t}$ in the $n$-th agent. Corollary \ref{cor:1} indicates that it is possible to improve a local policy's average performance on all agents by improving the local policy advantage.
In particular, by penalizing the local update with the TV distance, weighted by a coefficient ($\alpha$) proportional to the heterogeneity level, we may improve the performance of a local policy over all environments. In fact, it can be proved that optimizing the RHS of (\ref{eq:cor1}) will effectively improve the expected discounted reward on all agents. Please refer to Appendix D for more details.

The last two terms in (\ref{eq:cor1}) are penalty terms representing the expected TV distance and the maximum TV divergence between $\pi_{n}^{t+1}$ and $\pi_{n}^{t}$. The two penalties indicate that even if it is possible to improve the global performance by optimizing the local objective, the improvement only happens when the local policy $\pi_{n}^{t+1}$ is close to the global policy $\pi^{t}$. Given the first term on the RHS of (\ref{eq:cor1}) is
a constant, Corollary \ref{cor:1} suggests the following objective
for the local training 
\begin{alignat}{1}
h_n(\pi_{n}^{t+1};\pi^{t}) & = \mathbb{A}_{\pi^{t},\mu_{n},P_{n}}(\pi_{n}^{t+1}) - (\alpha+\beta)\nonumber\\
& \quad\cdot\mathbb{E}_{s\thicksim\rho_{\pi^{t},\mu_{n},P_{n}}} \left[D_{TV}(\pi^{t}(\cdot\vert s)\Vert\pi_{n}^{t+1}(\cdot\vert s))\right]\nonumber\\
& \quad-\delta D_{TV}^{\max}(\pi^{t},\pi_{n}^{t+1}).
\label{eq:prelocalobj}
\end{alignat}

\begin{rem}
\label{rem:ThePenalty} There are three penalty terms in (\ref{eq:prelocalobj}), namely the $\alpha$, $\beta$ and $\delta$ penalty term, penalizing the difference between the local and global policies. Thus, all of them can be regarded as a global penalty. However, they came from different sources and were added for different purposes. The $\alpha$ penalty occurred when investigating the connection between the local and global policy advantage, i.e., Theorem I, to tackle the data heterogeneity, while the $\beta$ and $\delta$ penalty came from 
(\ref{TRPOreturndiff}) and (\ref{eq:cporeturndiff}) to constrain the policy update at one agent.  
\end{rem}

\begin{rem}
\label{rem:The-IID-case} The IID case is covered by Corollary \ref{cor:1}. This special case occurs when all agents
share the same initial state distribution and environment dynamics, i.e. the level of heterogeneity $\left\Vert \mathbf{B}_{\pi,n}\right\Vert _{F}=0$
and hence $\alpha=0$ and $\beta$ becomes the constant $c$ in (\ref{TRPOObj}) for all agents. As a result, the $\alpha$ penalty term will disappear. However, the $\beta$ and $\delta$ penalty term is still valid. This indicates that, even with IID data, a global penalty term is needed, and data heterogeneity will require a higher global penalty.
\end{rem}

Next, we show the convergence of policy updates according to (\ref{eq:prelocalobj}). The theoretical analysis is motivated by the convergence analysis of Heterogeneous-Agent Trust Region Policy Optimization (HATPRO) algorithm \cite{kuba2022trust}. Similar to \cite{pmlr-v151-jin22a}, we focus our discussion on the tabular case, i.e. we assume the aggregation step can linearly combine local policies to obtain the global policy. A summary of this learning procedure is presented in Algorithm \ref{alg:PAvg}.
\begin{algorithm}[H]
\caption{\label{alg:PAvg}Federated Policy Iteration with Monotonic Improvement Guarantee. Tabular Case.}
\begin{algorithmic}[1]
	\REQUIRE{$\pi^0$}
	\FOR{round t = 0,1,...}
		\STATE{Synchronize the global model $\pi^{t}$ to every agents.}
		\STATE{Compute the exact values of $\alpha$, $\beta$ and $\delta$.}
		\FOR{agent k = 1,2,...,N}
			\STATE{Optimize $\pi_{k}^{t+1}=\arg\max_{\pi}\left[ h_{k}(\pi;\pi^{t})\right]$.}
			\STATE{Upload $\pi_{k}^{t+1}$ and $q_k$ to the central server.}
		\ENDFOR
		\STATE{The central server aggregates the $\theta$'s as $\pi^{t+1}(a \vert s) \leftarrow \sum_{k=1}^{N}q_{k}\pi_{k}^{t+1}(a \vert s), \forall s,a$.}
	\ENDFOR
\end{algorithmic}
\end{algorithm}
\begin{thm}
\label{thm:monotonicity} Algorithm \ref{alg:PAvg} is guaranteed to generate a monotonically improving sequence of policies, i.e. $\eta(\pi^{t+1}) \ge \eta(\pi^{t})$ for $t\in\mathbb{N}$.
\end{thm}
The proof is given in Appendix \ref{proof:monotonicity}. With Theorem \ref{thm:monotonicity}, the convergence of Algorithm \ref{alg:PAvg} is given by the following theorem.
\begin{thm}
\label{thm:2} 
A sequence of policies $\left(\pi^{t}\right)_{t=0}^{\infty}$ generated by Algorithm \ref{alg:PAvg} is convergent.
\end{thm}

The proof is given in Appendix \ref{proof:theorem2}. When policies are parameterized with tabular methods, Theorem \ref{thm:monotonicity} and \ref{thm:2} are exact. For non-linear parameterization, the algorithm can be similarly implemented, but the error term introduced by the aggregation step should be considered.

\section{The Proposed Algorithm: FedKL}
Although we have obtained a theory justifying the objective in (\ref{eq:prelocalobj}), directly optimizing it is difficult due to several issues. First, it is impractical to compute the last terms of (\ref{eq:prelocalobj}) by traversing all states. Moreover, despite it's theoretical importance, the TV distance does not have a closed-form expression by far. Second, the advantage function has to be estimated by sampling. Third, one agent cannot access information of other agents, making it impossible to determine the coefficients in (\ref{eq:prelocalobj}). In the following, we tackle these issues one by one.

\subsection{TV Distance} 
To handle the first issue, we replace the maximum of TV divergence over all states
$D_{TV}^{\max}(\pi^{t},\pi_{n}^{t+1})$ by its average value
$ \mathbb{E}_{s\thicksim\rho_{\pi^{t},\mu_{n},P_{n}}}\left[D_{TV}(\pi^{t}(\cdot\vert s)\lVert\pi_{n}^{t+1}(\cdot\vert s))\right]$. It has been shown that this heuristic approximation has similar empirical performance as the one with the maximum divergence \cite{DBLP:journals/corr/SchulmanLMJA15}. Furthermore, we replace the TV distance in (\ref{eq:prelocalobj})
by the KL divergence based on the Pinsker's inequality \cite{csiszar2011information}, i.e.
$D_{TV}(p\lVert q)\le\sqrt{\frac{1}{2}D_{KL}(p\lVert q)}$, and obtain
\begin{alignat}{1}
h_n(\pi_{n}^{t+1};\pi^{t}) \ge & \mathbb{A}_{\pi^{t},\mu_{n},P_{n}}(\pi_{n}^{t+1})-(\alpha+\beta+\delta) \label{eq:localobj}\\
& \mathbb{E}_{s\thicksim\rho_{\pi^{t},\mu_{n},P_{n}}} \left[\sqrt{\frac{1}{2}D_{KL}(\pi^{t}(\cdot\vert s)\lVert\pi_{n}^{t+1}(\cdot\vert s))}\right]\nonumber.
\end{alignat}

\subsection{Approximation of the Policy Advantage and Penalty Terms} 
To optimize (\ref{eq:localobj}), we need to calculate the policy advantage and the penalty terms. To this end, we assume all agents adopt the training algorithm proposed by \cite{DBLP:journals/corr/SchulmanWDRK17}. Next, we show how we can approximate (\ref{eq:localobj}) by Monte Carlo simulation. In each training round, the agent will perform $I$ iterations of local training. Let $\pi^{t+1}_{n,i}$ denote the local policy of the $n$-th agent after $i$ iterations of local training with $i=0,...,I$, where $\pi^{t+1}_{n,0} = \pi^t$ and $\pi^{t+1}_{n,I} = \pi^{t+1}_n$. 
Thus, the objective function for the $i$-th iteration of the $n$-th agent to obtain $\pi^{t+1}_{n,i}$ from $\pi^{t+1}_{n,i-1}$ can be given by 
\begin{alignat}{1}
& h_n(\pi_{n,i}^{t+1};\pi^{t},\pi_{n,i-1}^{t+1})\nonumber\\ 
=& \mathbb{E}_{s\thicksim\rho_{\pi^{t+1}_{n,i-1},\mu_{n},P_{n}},a\thicksim\pi_{n,i-1}^{t+1}} \left[\omega(s,a) A_{\pi^{t},P_{n}}(s,a)\right. \nonumber \\
& -c_{1} \sqrt{\frac{1}{2}D_{KL}^{}(\pi^{t} (\cdot\vert s) \lVert\pi_{n,i}^{t+1} (\cdot\vert s))}\nonumber\\
& \left.-c_{2} \text{\ensuremath{D_{KL}^{}}(\ensuremath{\pi_{n,i-1}^{t+1}} (\ensuremath{\cdot}\ensuremath{\vert}s) \ensuremath{\lVert\pi_{n,i}^{t+1}} (\ensuremath{\cdot}\ensuremath{\vert}s))}\right],\label{eq:finallocalobj}
\end{alignat}
where $c_{1}$ is used to approximate $(\alpha+\beta+\delta)$, the last term is the KL penalty introduced by the local training, and $c_2$ is used to approximate $c$ defined in (\ref{eq:trpo1}). Here, 
$\omega(s,a)=\frac{\rho_{\pi^{t},\mu_{n},P_{n}}(s)}{\rho_{\pi^{t+1}_{n,i-1},\mu_{n},P_{n}}(s)} \frac{\pi_{n,i}^{t+1}(a\ensuremath{\vert}s)}{\pi_{n,i-1}^{t+1}(a\ensuremath{\vert}s)}$ is the importance sampling ratio where we estimate $\mathbb{A}_{\pi^{t},\mu_{n},P_{n}}(\pi_{n}^{t+1})$ by states and actions sampled from $\rho_{\pi^{t+1}_{n,i-1},\mu_{n},P_{n}}$ and $\pi^{t+1}_{n,i-1}$.  The expected value of the first term in (\ref{eq:finallocalobj}) is an unbiased estimation of the policy advantage $\mathbb{A}_{\pi^{t},\mu_{n},P_{n}}(\pi_{n}^{t+1})$. However, it is difficult and inefficient to compute $\omega(s,a)$ directly. Therefore, we ignore the difference in the discounted visitation frequency and use $\omega(s,a)=\frac{\pi_{n,i}^{t+1}(a\ensuremath{\vert}s)}{\pi_{n,i-1}^{t+1}(a\ensuremath{\vert}s)},$ which yields a slightly biased estimation of the policy advantage but with smaller variance.

In the following discussion, we will refer to the second and third terms of (\ref{eq:finallocalobj}) as the global and local penalties, respectively. In particular, the second term penalizes the divergence of one agent's policy from the global policy and the third term penalizes the policy outputs in two consecutive iterations. Given both the two penalties are related to KL divergence, we will refer to the proposed algorithm as FedKL. 

\subsection{Adaptive Penalty Coefficients} 
In FRL, the initial state distribution and dynamics function of one agent is not known to others. Thus, the penalty coefficients $c_{1}$ and $c_{2}$ cannot be known locally. In this paper, we adopt the adaptive method in \cite{DBLP:journals/corr/SchulmanWDRK17} to determine the coefficients. Specifically, we introduce two hyper-parameters, $d_{local}$ and $d_{global}$, as the target local and global KL divergence. Then, $c_{1}$ and $c_{2}$ will be adjusted in an adaptive manner so that the target $d_{local}$ and $d_{global}$ are achieved in each policy update. After initializing $c_{1},c_{2}$, we perform a three-phases update in each policy update as follows:
\begin{itemize}
\item Perform policy update using E epochs of minibatch gradient ascent to optimize
$h_n(\pi_{n,i}^{t+1};\pi^{t},\pi_{n,i-1}^{t+1})$
in (\ref{eq:finallocalobj}).
\item Compute\\$d = \mathbb{E}_{s\thicksim\rho_{\pi^{t+1}_{n,i-1},\mu_{n},P_{n}}}\left[D_{KL} \left(\pi_{n,i-1}^{t+1}(\ensuremath{\cdot}\ensuremath{\vert}s)\lVert\pi_{n,i}^{t+1}(\ensuremath{\cdot}\ensuremath{\vert}s)\right)\right]$
\begin{itemize}
\item $\text{If }d<d_{local}/1.1,c_{2}\leftarrow c_{2}/2$
\item $\text{If }d>d_{local}\times1.1,c_{2}\leftarrow c_{2}\times2$
\end{itemize}
\item Compute $d = \sqrt{\frac{1}{2}D_{KL}^{}(\pi^{t} (\cdot\vert s) \lVert\pi_{n,i}^{t+1} (\cdot\vert s))}$
\begin{itemize}
\item $\text{If }d<d_{global}/1.1,c_{1}\leftarrow c_{1}/2$
\item $\text{If }d>d_{global}\times1.1,c_{1}\leftarrow c_{1}\times2$
\end{itemize}
\end{itemize}
Specifically, in each training round, FedKL performs several iterations of SGD update. After each iteration, FedKL updates its estimation of $c_{1}$ and $c_{2}$ for the next iteration so that the target local and global KL divergence are achieved. 

\begin{algorithm}[H]
\caption{\label{alg:FedPG}FedKL using adaptive penalty coefficients.}
\begin{algorithmic}[1]
	\REQUIRE{$K,I,T,E,\pi^0$}
	\FOR{round t = 0,1,...}
		\STATE{Selects a subset of K agents according to some criteria.}
		\STATE{Synchronize the global model to every selected agents.}
		\FOR{agent k = 1,2,...,K}
			\FOR{iteration i = 1,2,...,I}
				\STATE{Run policy $\pi_{k,i-1}^{t+1}$ in environment for T timesteps.}
				\STATE{Optimize $h_k(\pi_{k,i}^{t+1};\pi^{t},\pi_{n,i-1}^{t+1})$ for $E$ epochs to obtain $\pi_{k,i}^{t+1}$.}
			\ENDFOR
			\STATE{Upload $\theta_{k}^{t+1}$ and $l_k$ to the central server.}
		\ENDFOR
		\STATE{The central server aggregates the $\theta$'s as $\theta^{t+1}=\sum_{k=1}^{K}\frac{l_k}{L}\theta_{k}^{t+1}$.}
	\ENDFOR
\end{algorithmic}
\end{algorithm}
\begin{rem}
A larger target KL divergence ($d_{global}$ and $d_{local}$) will lead to a smaller penalty coefficient on the KL divergence ($c_1$ and $c_2$). Given the KL divergence between two policies before and after training represents the learning step size, the target KL divergence can be regarded as a target learning step size. Thus, $d_{global}$ and $d_{local}$ can help maintain a constant step size for each training round and iteration, respectively.

\end{rem}
The heuristics for tuning $d_{local}$ and $d_{global}$ will be discussed in the experiments section. The training procedure of FedKL is illustrated in Algorithm \ref{alg:FedPG}.

\section{Experiments\label{sec:Experiments}}

\begin{figure*}[!t]
\centering
\subfloat[]{\includegraphics[width=0.48\columnwidth]{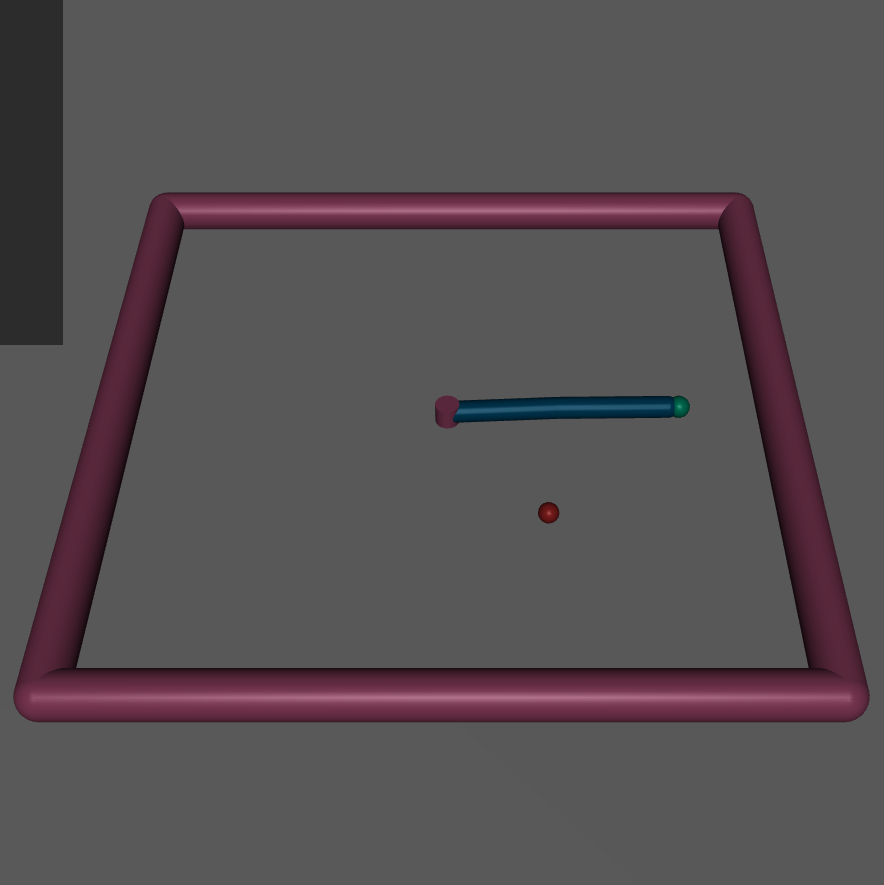}
\label{fig:reacherv2}}
\hfil
\subfloat[]{\includegraphics[width=0.48\columnwidth]{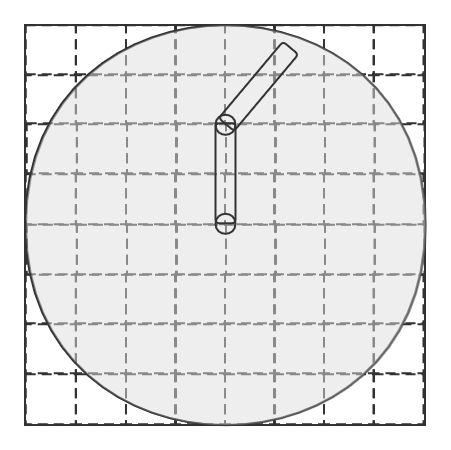}
\label{fig:reacher}}
\caption{Reacher's field splitting. Sub-fields in grey are reachable by the robotic arm. Sub-environments are created based on these sub-fields.}
\label{fig:reachers}
\end{figure*} 

In the experiments, we will compare the proposed FedKL with other algorithm-level solutions for tackling the heterogeneity issue, including FedAvg, FedProx and FMARL. In particular, agents of the FedAvg system will perform local training utilizing the algorithm proposed in \cite{DBLP:journals/corr/SchulmanWDRK17}, with the local penalty term. Agents of the FedProx system will deploy the same algorithm as FedAvg and add the proximal term \cite{MLSYS2020_38af8613}. Agents of the FMARL system will deploy the same algorithm as FedAvg and utilize the gradient decay scheme in \cite{DBLP:journals/corr/abs-2103-13026}. On the other hand, agents of the proposed FedKL system will implement the same algorithm as FedAvg, but with the global KL divergence penalty. All code, experiments and modified environments can be found at: https://github.com/ShiehShieh/FedKL.

\subsection{Experiment Platforms}
To the best of authors' knowledge, there are few works focusing on the
data heterogeneity issue in FRL, making it difficult to find
benchmarks for comparison purpose. Therefore, we made several light-weight modifications to popular RL simulators to accommodate the FRL setting. There are two
widely used simulation environments for evaluating the performance
of RL controllers, namely the MuJoCo simulator \cite{6386109,DBLP:journals/corr/BrockmanCPSSTZ16}
and the Flow simulator \cite{pmlr-v87-vinitsky18a}. In the following, we introduce the two simulators and the settings for FRL experiments.

\subsubsection{Heterogeneous Robotic Tasks}
In Experiment 1, we consider a robotic task in the Reacher-v2 environment provided by OpenAI Gym. As shown in Fig. \ref{fig:reacherv2}, the agent actuates a robotic arm to move the end-effector to reach a target, which is denoted by a red sphere and spawn at a random position in a $0.4\times0.4$ field.
The environment consists of a 11-dimensional state space, a 2-dimensional action space and a reward function that determines the reward based on the length
of the action vector as well as the distance between the end-effector and the target.

To simulate the heterogeneity due to different initial state distributions, we split the field into $Q$ sub-fields and create $Q$ sub-environments.  For each sub-environment, we assume the robotic arm will always start from one of the $Q$ specific sub-fields and try to reach the target. Note that this corresponds to an extreme setup for the different initial state distributions. The case with $Q=60$ is illustrated in Fig. \ref{fig:reacher}. To simulate the
heterogeneity due to different environment dynamics, we add noise to each action as in \cite{Lim2021FederatedRL}. Specifically, we add noise to distort all actions such that they are different for different devices. Let $m_{noise}$ and $\sigma_{noise}$ denote the mean and variance of the added noise. We assume $m_{noise}=0$ and pick different variances to represent different levels of heterogeneity. 


\subsubsection{Traffic Control by Isolated RL Agents}
In Experiment 2, we consider the traffic control in the ``figure eight'' network  \cite{pmlr-v87-vinitsky18a}. As shown in Fig. \ref{fig:vps}, there are in total 14 vehicles on the road with half of them controlled by RL agents and the rest controlled by the underlying Intelligent Driver Model (IDM) controller provided by the SUMO library. The IDM controller can be regarded as the human-driver model. This experiment is designed for mixed-autonomy, where a centralized controller coordinates a set of connected autonomous vehicles (CAVs) to improve the system-level velocities.

\begin{figure*}[!t]
\centering
\subfloat[]{\includegraphics[width=0.48\columnwidth]{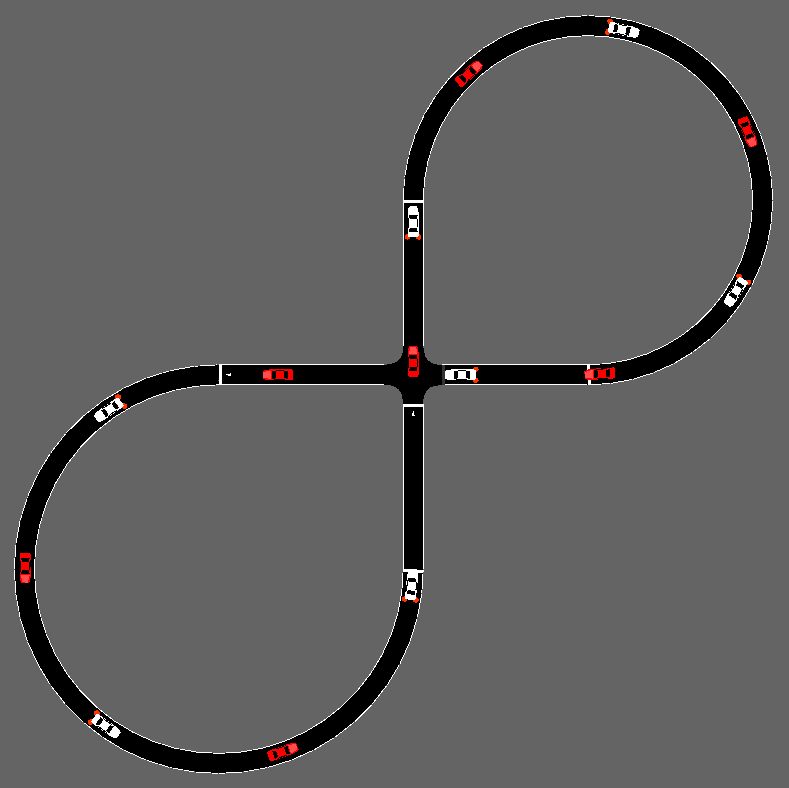}
\label{fig:schema1}}
\hfil
\subfloat[]{\includegraphics[width=0.48\columnwidth]{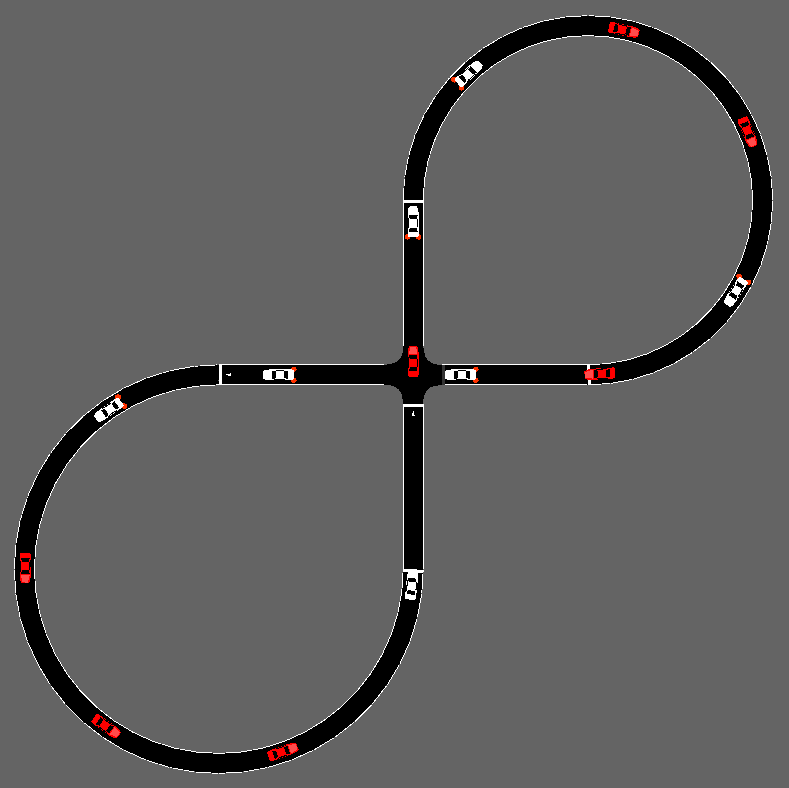}\label{fig:schema2}}
\caption{Vehicle Placement Schema. Red cars are human-driven vehicles. White cars are RL-controlled vehicles.}
\label{fig:vps}
\end{figure*} 

Some modifications are required to run FRL experiments in the above network. Fortunately, there are a few modifications proposed by \cite{DBLP:journals/corr/abs-2103-13026}, where autonomous vehicles (AVs) are not connected by a centralized controller but isolated agents. Each RL-controlled vehicle is able to observe the position and speed of itself, and that for the two vehicles ahead and behind it. As a result, each state is represented by a 6-dimensional vector. Since each agent is controlled by its local policy, every action is a scalar specifying the acceleration for the current time-step. Further assume that the average speed of all vehicles is accessible to all AVs.

By default, AVs and human-driven vehicles are positioned alternatively,
as shown in Fig. \ref{fig:schema1}. If we utilize `h' and `r' to represent the human-driven vehicle and AVs, respectively, the default placement can be written as [h,r,h,r,h,r,h,r,h,r,h,r,h,r]. To simulate the heterogeneity due to different initial state distributions, we fix the starting position of all AVs such that everyone starts from the same position in each run. 
Given the AVs have access to the information about the two vehicles before and after, we may create different environment dynamics by adjusting the placement pattern. For example, the car placement [h,r,h,h,r,r,h,r,h,h,h,r,r,r], as shown in Fig. \ref{fig:schema2}, will lead to different environment dynamics.

\subsubsection{Implementation Details}
For both experiments, we use neural networks to represent policies as in \cite{DBLP:journals/corr/SchulmanLMJA15,pmlr-v87-vinitsky18a}. Specifically, fully-connected Multilayer Perceptrons (MLPs) with $\tanh$ non-linearity are utilized for two experiments with  hidden layers (64, 64) and (100, 50, 25) \cite{pmlr-v87-vinitsky18a}, respectively. Generalized Advantage Estimation (GAE) \cite{DBLP:journals/corr/SchulmanMLJA15} is used to estimate the advantage function. The value functions for two experiments are represented by two MLPs with $\tanh$ non-linearity and hidden layers (64, 64) and (256, 256), respectively. 
The learning rate of one algorithm will remain the same for one simulation without decaying, but it will be tuned for different simulations. The hyperparameters $d_{local}$, $d_{global}$ and $\mu$\footnote{This is a hyperparameter defined by FedProx to control the weight of the proximal term.} are carefully tuned so that they are near-optimal. In each training round, we set $K=3,I=20$, and $T=2048$ for Experiment 1, and $K=7,I=50$, and $T=1500$ for Experiment 2. For both experiments, we set the discount factor $\gamma$, batch size, and the discount factor for GAE, $\lambda$ in \cite{DBLP:journals/corr/SchulmanMLJA15}, as 0.99, 64, and 0.95, respectively.

\subsection{Level of Heterogeneity and Its Impact}

\begin{figure}[!t]
\centering
\includegraphics[width=1.0\columnwidth]{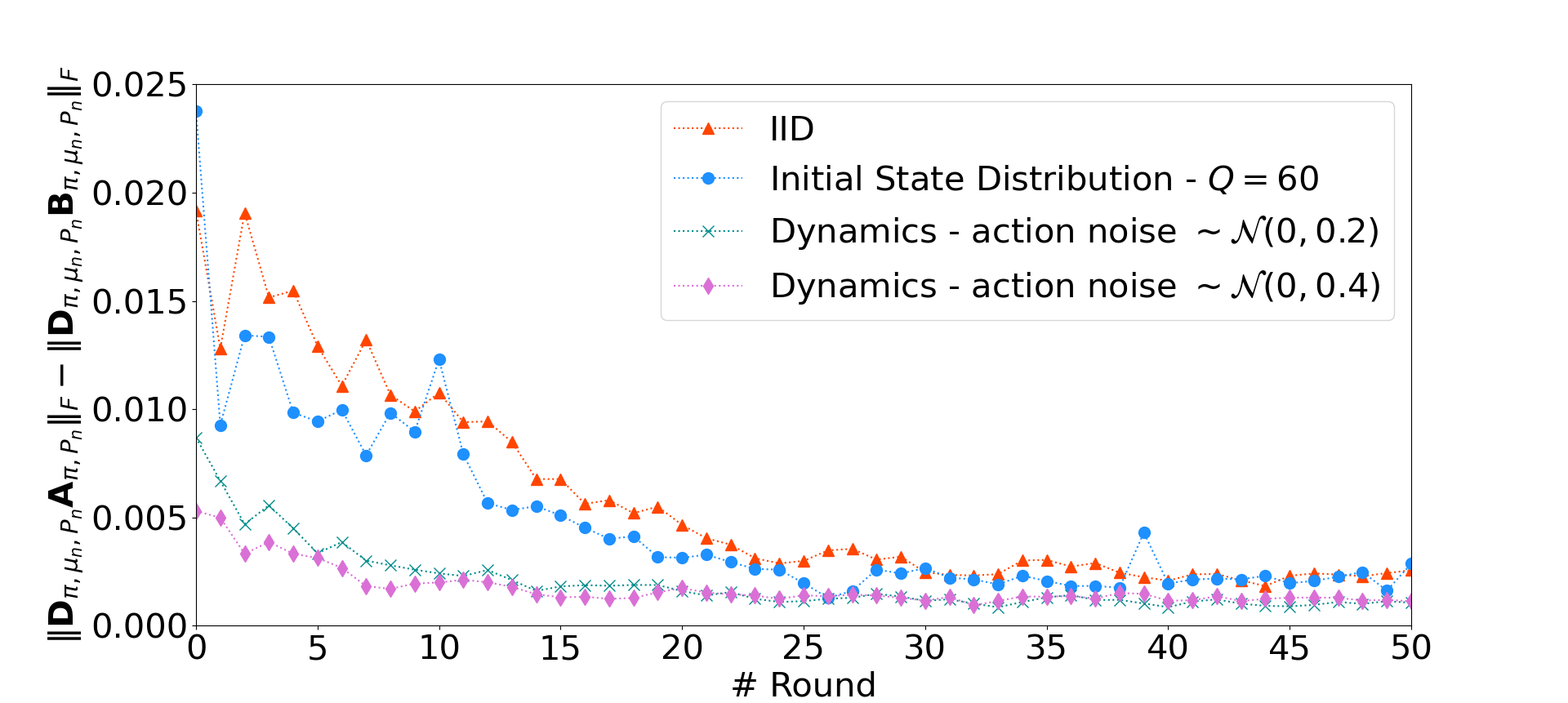}
\hfil
\caption{Visualization of $\left\Vert\mathbf{D}_{\pi,\mu_{n},P_{n}} \mathbf{A}_{\pi,P_{n}}\right\Vert_{F}-\left\Vert\mathbf{D}_{\pi,\mu_{n},P_{n}}\mathbf{B}_{\pi,\mu_{n},P_{n}}\right\Vert_{F}$ on modified Reacher-v2 environments. Results are averaged across three runs.}
\label{fig:heterogeneity_level}
\end{figure}

\begin{figure}[!t]
\centering
\includegraphics[width=1.0\columnwidth]{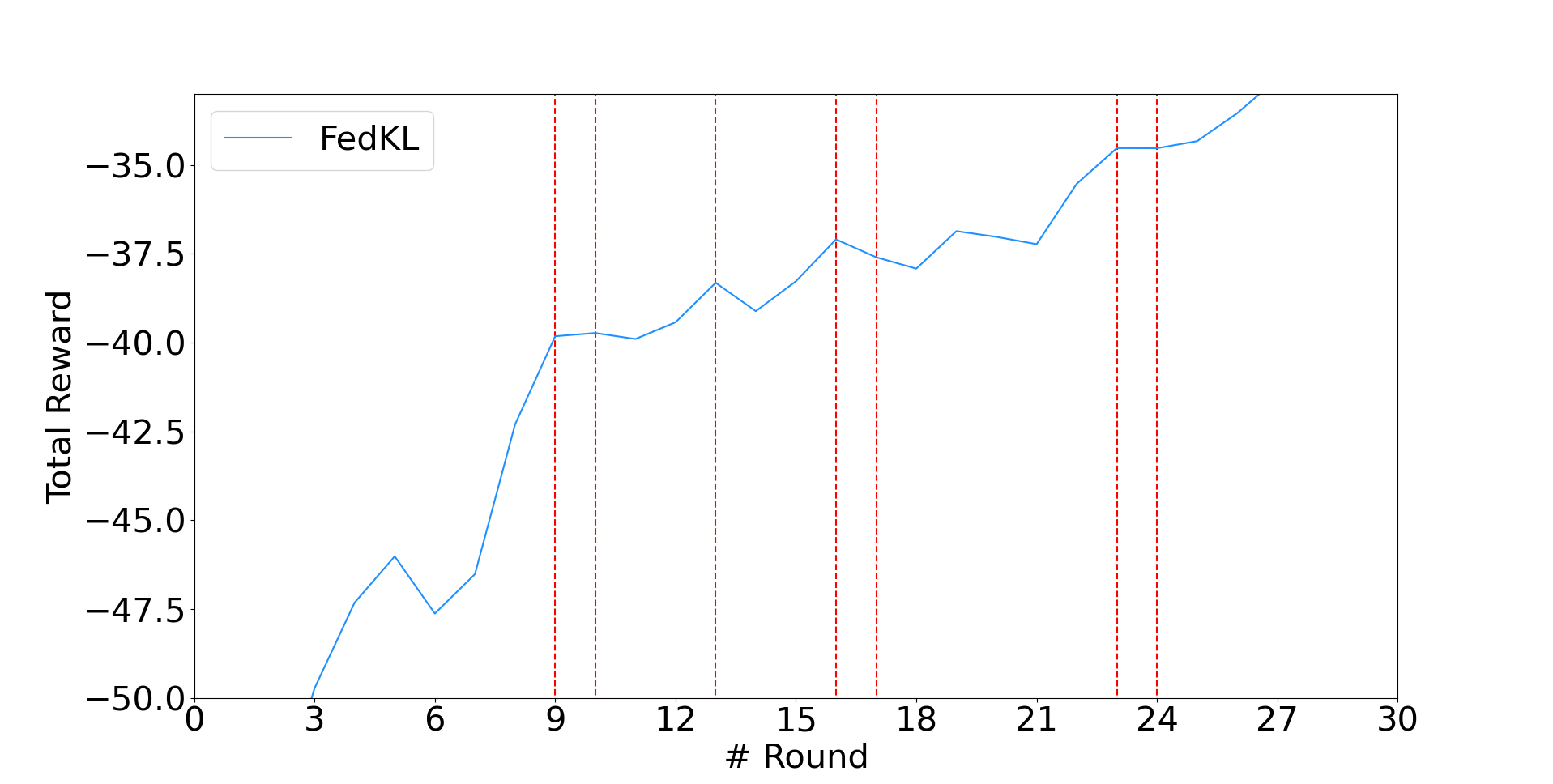}
\hfil
\caption{The effect of negative $G$  on modified Reacher-V2 environments with different environment dynamics (action noise $\thicksim\mathcal{N}(0, 0.4)$). Red lines indicate the rounds where $G<0$ happens.}
\label{fig:effect_neg_da-b}
\end{figure}

In this section, we will investigate the type and level of heterogeneity and how they affect the performance of FRL. For ease of illustration, we will use Experiment 1 to analyze the heterogeneity type and level. Detailed performance comparison between different algorithms and in-depth discussion regarding the key parameters will be given in the next section based on results from Experiment 2.

\begin{figure}[!t]
\centering
\includegraphics[width=1.0\columnwidth]{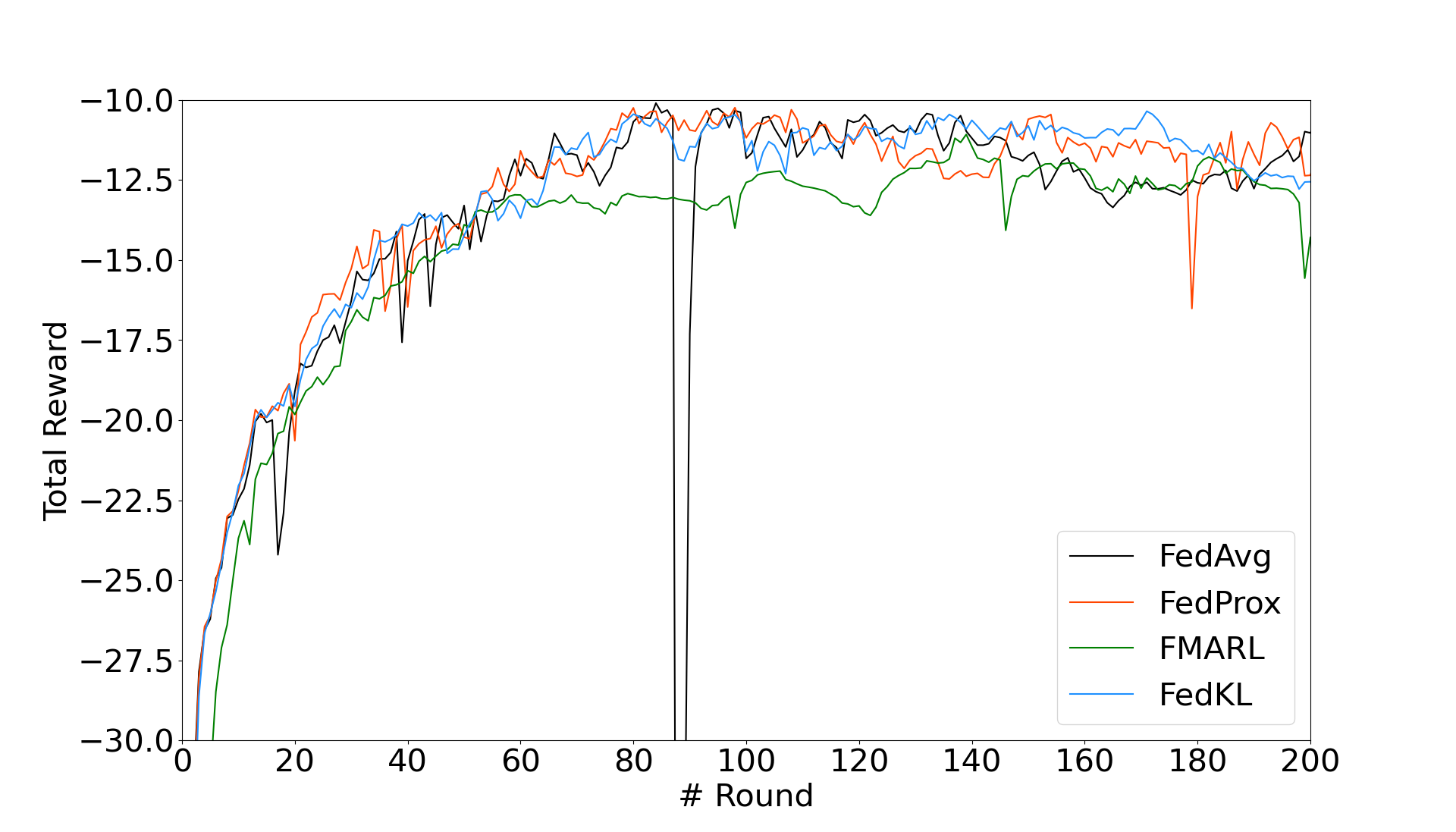}
\hfil
\caption{Comparison of FedAvg, FedProx and FedKL on modified Reacher-V2 environments with different initial state distributions $Q=60$. For all algorithms, the learning rate is 0.001. For FedProx, $\mu=0.02$. For FMARL, $\lambda=0.9999$. For FedKL, $d_{global}=0.6$. Results are averaged across three runs.}
\label{fig:experiment_mujoco_initstate}
\end{figure}

\begin{figure}[!t]
\centering
\includegraphics[width=1.0\columnwidth]{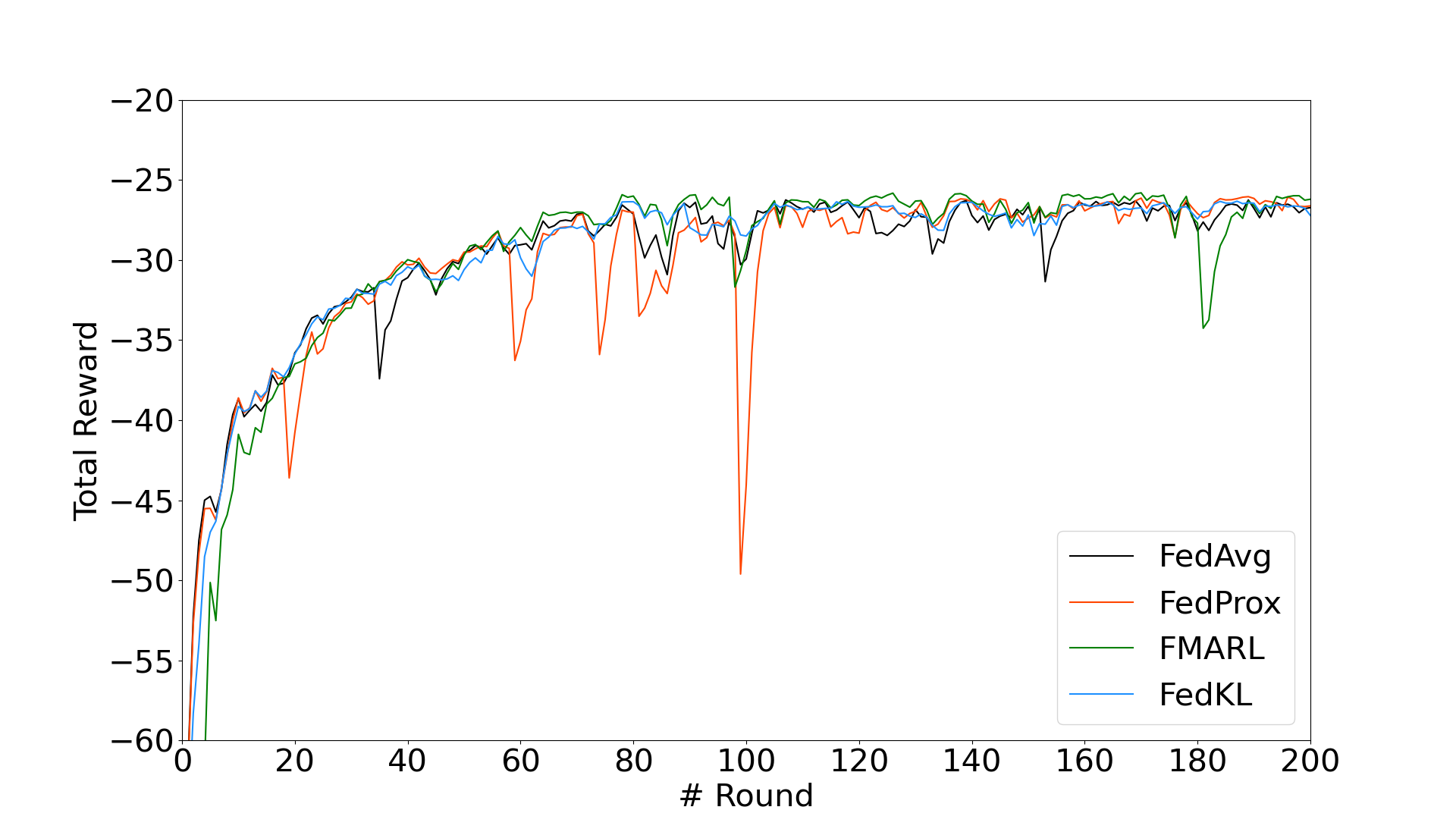}
\hfil
\caption{Comparison of FedAvg, FedProx and FedKL on modified Reacher-V2 environments with different environment dynamics (action noise $\thicksim\mathcal{N}(0, 0.4)$). For all algorithms, the learning rate is 0.001. For FedProx, $\mu=0.02$. For FMARL, $\lambda=0.9999$. For FedKL, $d_{global}=0.6$. Results are averaged across three runs.}
\label{fig:experiment_mujoco_dynamics}
\end{figure}

To illustrate the impact of data heterogeneity, we show in Fig. \ref{fig:heterogeneity_level} the changes of $G = \left\Vert \mathbf{D}_{\pi,\mu_{n},P_{n}} \mathbf{A}_{\pi,P_{n}}\right\Vert_{F} - \left\Vert \mathbf{D}_{\pi,\mu_{n},P_{n}} \mathbf{B}_{\pi,\mu_{n},P_{n}}\right\Vert _{F}$. Each term in (\ref{BleD1}) is rescaled by $\mathbf{D}_{\pi,\mu_{n},P_{n}}$ to overcome the numerical instability introduced by the inverse operation $\mathbf{D}_{\pi,\mu_{n},P_{n}}^{-1}$. Note that according to Corollary \ref{cor}, the local training is useful only when $\left\Vert \mathbf{A}_{\pi,P_{n}}\right\Vert_{F} - \left\Vert \mathbf{B}_{\pi,\mu_{n},P_{n}}\right\Vert _{F}>0$. Thus, we may regard $G$ as the potential of one agent to contribute to the global policy. In particular, if we treat $\left\Vert \mathbf{D}_{\pi,\mu_{n},P_{n}} \mathbf{A}_{\pi,P_{n}}\right\Vert_{F}$ as the local training potential, then $G$ is the local potential penalized by the heterogeneity level, i.e., even if there is potential improvement from local training, it may not be helpful for the global policy when the heterogeneity level is high. We consider four different scenarios, including the IID case, the case with different initial state distributions ($Q=60$), and two cases with different level of environment dynamics, where the heterogeneity level is controlled by the variance of the action noises. It can be observed that $G$ decreases with training because training will reduce agents' potential to further contribute to the global policy. On the other hand, comparing the two cases with different environment dynamics, we notice that higher level of heterogeneity will lead to smaller $G$, indicating that heterogeneity will limit the potential contribution of one agent. Note that the comparison between different types of heterogeneity may not be meaningful, as their local advantage function may have different ranges. Fluctuations of the curves are due to limited number of simulations and matrix estimation error.

To further illustrate the impact of the heterogeneity level, we show in Fig. \ref{fig:effect_neg_da-b} the training performance with high level of action noise. The red lines indicate the training rounds where $G<0$ happens in selected agents. It can be observed that, for those rounds, the global training will slow down, which agrees with Corollary \ref{cor}, i.e., $G<0$ indicates that the associate training is not helpful for the global policy. In Figs. \ref{fig:experiment_mujoco_initstate} and \ref{fig:experiment_mujoco_dynamics}, we show the performance comparison of FedAvg, FedProx, FMARL and FedKL with different initial state distributions and different transition probabilities, respectively. It can be observed that with the same learning step size, FedKL can converge to a stable performance, but both FedAvg and FedProx may diverge or fluctuate. We can certainly reduce the step size for FedAvg and FedProx, but this will lead to slow learning speed. See next section for more related discussions.  

\subsection{Advantage of the KL-based Global Penalty}

\begin{figure}[!t]
\centering
\includegraphics[width=1.0\columnwidth]{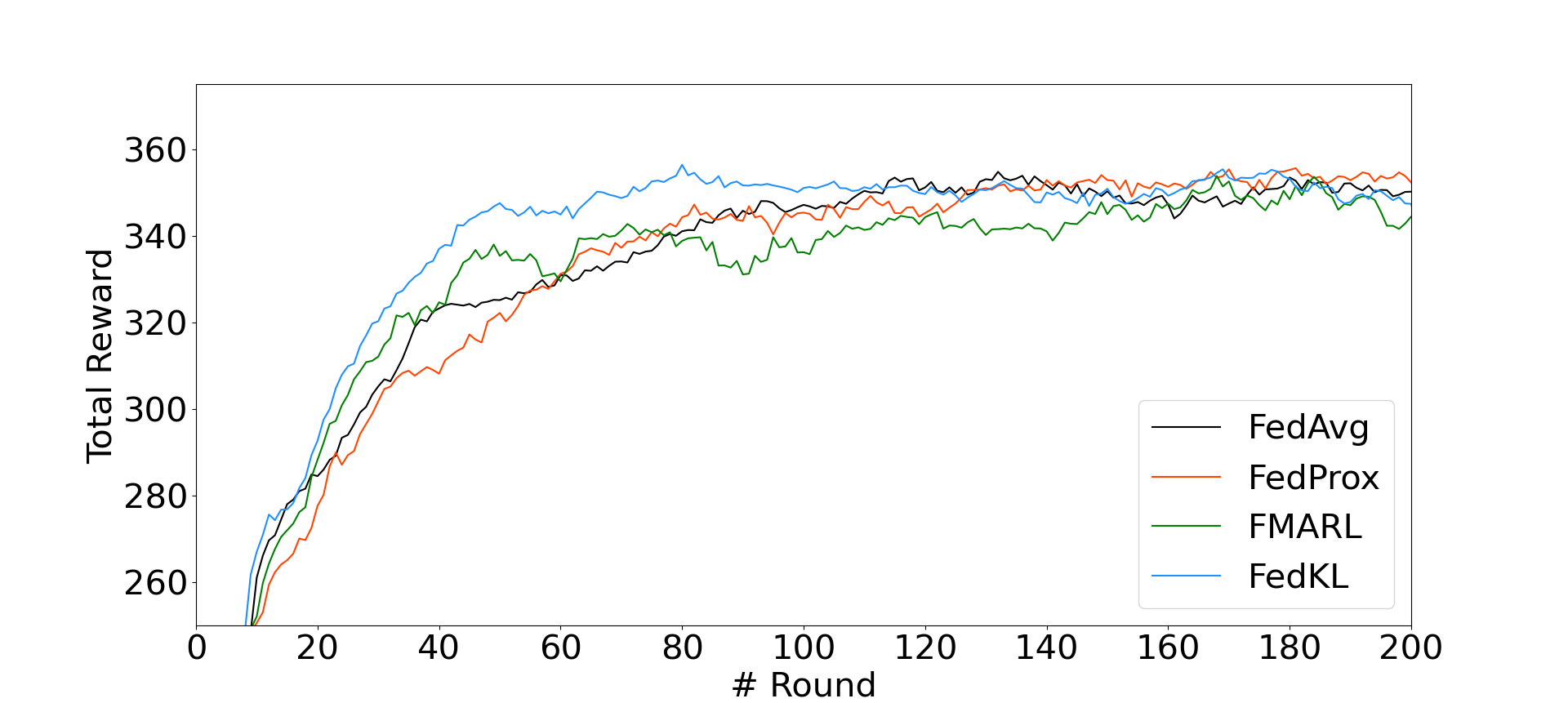}
\caption{Learning curves of FedAvg, FedProx and FedKL for Experiment 2. The best performance is shown for each algorithm. For all algorithms, the learning rate is 0.01. For FedAvg, $d_{local}=0.0002$. For FedProx, $d_{local}=0.0002,\mu=0.001$. For FedKL, $d_{local}=0.0003,d_{global}=0.15$. Results are averaged across three runs.}
\label{fig:experiment_flow}
\end{figure} 

\begin{figure}[!t]
\centering
\includegraphics[width=1.0\columnwidth]{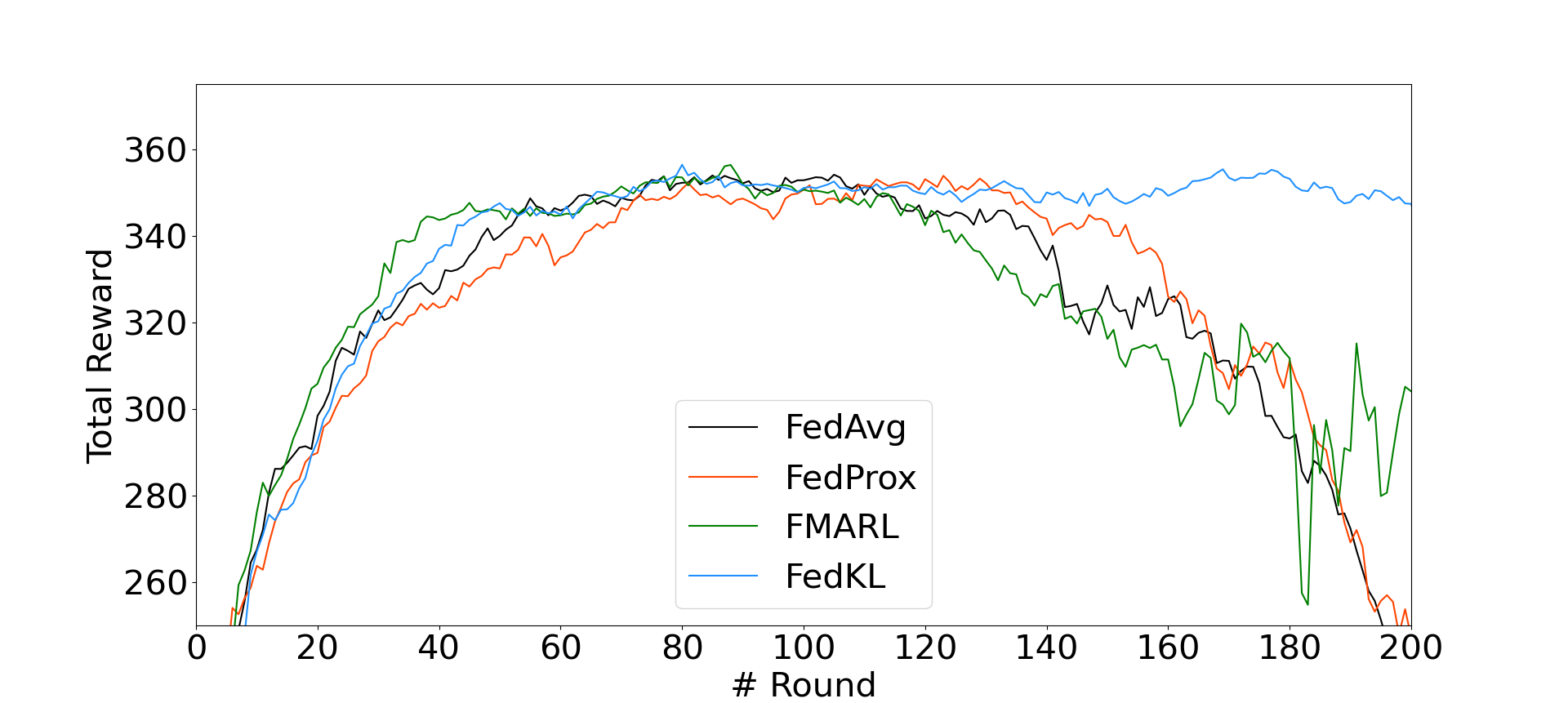}
\caption{The settings are the same as Fig. \ref{fig:experiment_flow}, but all algorithms have same local penalty, i.e. $d_{local}=0.0003$.}
\label{fig:experiment_flow_same_params}
\end{figure} 

In the following, we show the results of Experiment 2 where both types of heterogeneity are present. The best performance of each algorithm is shown in Fig. \ref{fig:experiment_flow}, where the averaged total reward over three runs is reported. Each run uses the same seed for the SUMO simulator but different seeds for the parameter initialization of the neural network. It can be observed that FedKL converges much faster than FedAvg, FedProx and FMARL. The faster convergence of FedKL indicates a larger step size. Then, a natural question is whether we can increase the step size of FedAvg, FedProx and FMARL to speed up their training.  In Fig. \ref{fig:experiment_flow_same_params}, we show the results where the $d_{local}$ for FedAvg, FedProx and FMARL are increased. As a result, all of them converge at a similar speed as FedKL, but they diverge after reaching the maximum reward. The above results provide empirical evidence that directly constraining the model output by KL divergence is more effective in handling the data heterogeneity than penalizing model divergence in the parameter space. Furthermore, applying a proper global KL penalty can prevent divergence from happening without loss of learning speed. 

{\bf{Why KL Penalty is better?}} It can be observed from Fig. \ref{fig:experiment_flow_same_params} (after Round 120) that the proximal term of FedProx can slow down the divergence, but failed in avoiding it. This is because the proximal term $\left\Vert\theta_{n}^{t+1}-\theta^{t}\right\Vert$ constrains the learning in the model parameter space by the $\mathcal{L}2$-norm. Although the change in the model parameter space will lead to changes in the output distribution space (actions) and eventually influence the training objective, it is more effective to directly constrain the training in the distribution space \cite{kakade2001natural,bagnell2003covariant}, which justifies the performance advantage of the proposed KL-divergence based penalty.  

\begin{figure*}[!t]
\centering
\includegraphics[width=1.2\columnwidth]{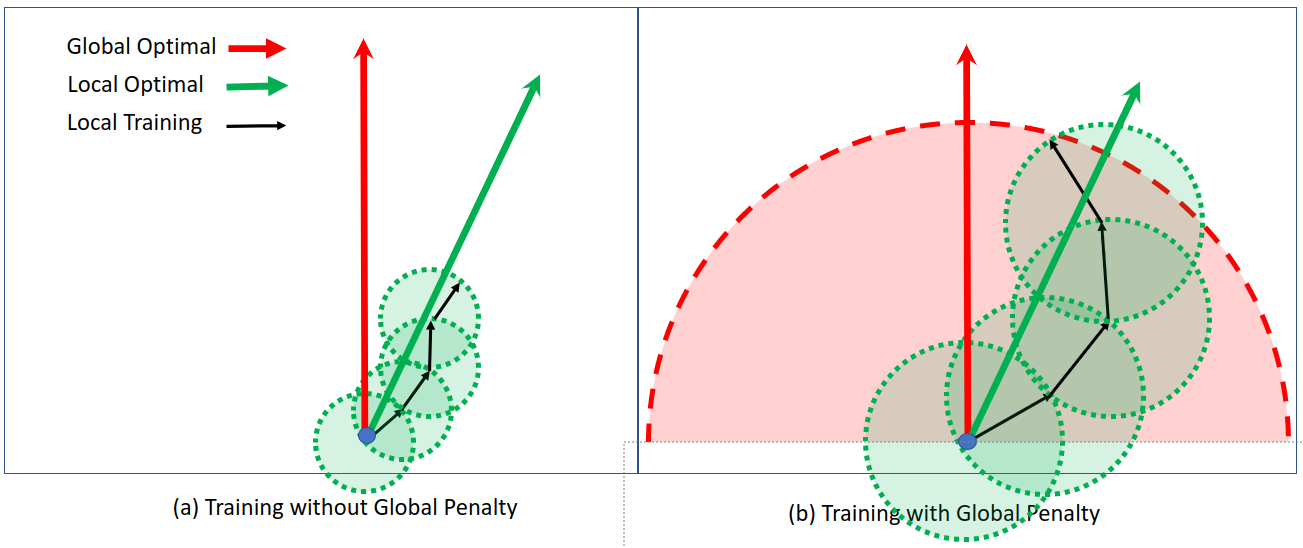}
\label{fig:d_local}
\label{fig:d_global}
\caption{The effect of $d_{global}$ and $d_{local}$. (a) Without the global penalty, agents must use a small $d_{local}$ at the cost of slow learning. (b) With the global penalty, agents can explore a wider region (faster learning) without worrying about divergence.}
\label{fig:d_local_d_global}
\end{figure*} 

{\bf{Why Global Penalty Helps?}}  We now investigate how $d_{global}$ helps speed up convergence and stabilize training. Assume there are $I$ iterations of local training in each round and let $d_i$ denote the policy update in the $i$-th iteration. We consider two cases, without and with the global penalty (target step size), respectively, as illustrated in Fig. \ref{fig:d_local_d_global}. We use the red and green arrows to represent the optimal optimization direction of the global and local objectives, respectively. Note that they are not on the same direction due to data heterogeneity.    
For the first case, we only apply the local penalty which constrains the step size for each iteration with $|d_i| \le d_{local}$, i.e., the radius of the green circle in part (a) of Fig. \ref{fig:d_local_d_global}. Let $d_{max}$ denote the highest tolerable step size in one training round, without causing training divergence. Thus, we need $d_{local} \le d_{max}/I$, which is determined by considering the worst case scenario where all updates are on the same direction. However, the worst case happens with a very low chance. As a result, $d_{local}$ is pessimistically small for most training. Performance of FedAvg in Fig. \ref{fig:experiment_flow} corresponds to the case that $d_{local}$ is small enough to avoid divergence, but causes slow convergence. Performance of FedAvg in Fig. \ref{fig:experiment_flow_same_params} corresponds to the case that $d_{local}$ is large enough for fast training, but causes divergence. 

Now consider the second case with $d_{global}$, denoted by the radius of the red circle in part (b) of Fig. \ref{fig:d_local_d_global}. Under such circumstances, we only need to guarantee $\left| \sum_{i}^I d_i \right| \le d_{global} \le d_{max}$ and $\left|d_i\right| \le d_{local}$.  As a result, $d_{local}$ can be much larger than $d_{max}/I$, leading to a much larger per-iteration step size, represented by the larger radius of the green circle in part (b) of Fig. \ref{fig:d_local_d_global}. Thus, adding $d_{global}$ enables a larger $d_{local}$ and their joint efforts can achieve a better trade-off between training speed (step size) and convergence.  

{\bf{Heuristics for Tuning Target Step Size:}}
The global KL divergence is a good indicator of model divergence and we usually observe a large global KL divergence followed by a drastic performance degradation. This observation suggests a heuristic way for tuning the global step-size target $d_{global}$. For example, one can first run the algorithm without activating the global KL penalty (set $c_{1}$ to 0). Then, one may find the optimal value of $d_{local}$ and observe the average value of the KL divergence between $\pi^{t}$ and $\pi^{t+1}_{n}$ for certain number of rounds. It was observed that using a $d_{global}$ slightly larger than the average, together with a $d_{local}$ larger than the optimal value found in previous runs, can significantly speed up learning in the early stage of training. 

\subsection{Communication-Efficient FRL}
There are many benefits to solving the data heterogeneity issue, including faster convergence, stable performance, etc. Faster convergence will also lead to fewer rounds of communication and mitigate the communication bottleneck issue for FRL. For example, to achieve a total reward of 350 in Experiment 2, it needs around 110 rounds of communication for both FedAvg and FedProx, but only 50 rounds for FedKL, representing a saving of more than $50\%$ of the communication workload.

Local agents in FedKL use the Actor-Critic architecture \cite{NIPS1999_6449f44a,NIPS1999_464d828b}, where the actor refers to the policy $\pi$ and the critic refers to the value function or the advantage function. One possible way to reduce communication workload is to keep the critic local since it is not needed for inference. We will investigate this idea in future works.

Another issue for large-scale FRL applications like autonomous driving is the model aggregation over many agents moving in a large area. A possible solution is the hierarchical FL scheme \cite{9148862}, where aggregations happen in two layers. First, local aggregations among agents within a small region are performed on a local server. Then, the aggregation results from multiple local servers are sent to a global server for global aggregation. The local and global aggregations can happen with different frequencies, which may significantly reduce the communication workload. Unfortunately, the convergence analysis for the hierarchical scheme with data heterogeneity is not available, and it will be interesting to investigate the performance of FedKL under such a hierarchical framework.



\section{Conclusion}

In this paper, we investigated the data heterogeneity issue in FRL systems. For that purpose, we first classified the types of data heterogeneity in FRL and defined the heterogeneity level. It was proved that although optimizing the local policy may improve the global policy, the deviation of the local policy from the global one must be properly constrained. A necessary condition for the local policy improvement to be beneficial for the global objective was also derived with respect to the proposed heterogeneity level. Based on the theoretical result, a KL-divergence based global penalty term was proposed, which was shown to be able to enable a larger local step size without causing model divergence. The convergence proof of the proposed algorithm is also provided. Experiment results demonstrated the advantage of the proposed scheme in speeding up and stabilizing the training. One conclusion we can draw from this work is that heterogeneity not only reduces the contribution of one agent's local training to the global policy, but also increases the risk of divergence. It is thus important to balance between learning speed and convergence with a proper learning step size. Due to the lack of global information in the agents, a feasible solution is to maximize the learning speed and avoid divergence by jointly tuning the local and global penalty. In the future, a big challenge for implementing large-scale FRL is the model aggregation among many moving agents and it is critical to develop distributed aggregation schemes with heterogeneous data.


%
{\appendices
\section{Proof of Theorem \ref{thm:1}\label{proof:theorem1}}
\begin{proof}
Following \cite{baumgartner2011inequality}, we define the ``absolute'' value of a matrix $\mathbf{X}$ as
\begin{eqnarray}
\label{absolute}
\left| \mathbf{X} \right| = (\mathbf{X}^{\dagger}\mathbf{X})^{\frac{1}{2}}, \qquad
\left|\mathbf{X}^{\dagger}\right| = (\mathbf{X} \mathbf{X}^{\dagger})^{\frac{1}{2}},
\end{eqnarray}
where $\mathbf{X}^{\dagger}$ denotes the conjugate transpose of $\mathbf{X}$. In other words, $\left| \mathbf{X} \right|$ can be regarded as the square-root of $\mathbf{X}^{\dagger}\mathbf{X}$. 
For a real matrix $\mathbf{X}$, $\mathbf{X}^{T} \mathbf{X}$ is symmetric and positive semi-definite. Thus, there must be a unique square root $\left| \mathbf{X} \right|$ of $\mathbf{X}^{T} \mathbf{X}$ that is real and positive semi-definite. For a given scalar $x$, the above operator will degenerate to the absolute value operator $\left| x \right| = (xx)^{\frac{1}{2}}$. 

To understand the impact of the local optimization on the global objective, we define the gap between the global and local policy advantage for taking the updated local policy of the $n$-th agent $\pi_{n}^{t+1}$ as 
\begin{eqnarray}
\label{G}
G(\pi_{n}^{t+1}) = \left|\sum_{k=1}^{N}q_{k}\mathbb{A}_{\pi^{t},\mu_k, P_{k}}(\pi_{n}^{t+1})-\mathbb{A}_{\pi^{t},\mu_n,P_{n}}(\pi_{n}^{t+1})\right|.
\end{eqnarray}
Utilizing the vector notation from (\ref{eq:vectorpolicyadvantage}), we can calculate the gap as
\begin{alignat}{1}
 & G(\pi_{n}^{t+1})\nonumber\\
 = & \left| \mbox{Tr} \left(\sum_{k=1}^{N} q_{k} \mathbf{D}_{\pi^{t},\mu_{n},P_{n}} \mathbf{D}^{-1}_{\pi^{t},\mu_{n},P_{n}} \mathbf{D}_{\pi^{t},\mu_{k},P_{k}} \mathbf{A}_{\pi^{t},P_{k}} \mathbf{\Pi}_{\pi_{n}^{t+1}}^{T} \right.\right.\nonumber\\
 & \left.\left. -\mathbf{D}_{\pi^{t},\mu_{n},P_{n}} \mathbf{A}_{\pi^{t},P_{n}} \mathbf{\Pi}_{\pi_{n}^{t+1}}^{T} \right)\right| \\
 = & \left| \mbox{Tr} \left(\sum_{k=1}^{N} q_{k} \mathbf{\Pi}_{\pi_{n}^{t+1}}^{T} \mathbf{D}_{\pi^{t},\mu_{n},P_{n}} \mathbf{D}^{-1}_{\pi^{t},\mu_{n},P_{n}} \mathbf{D}_{\pi^{t},\mu_{k},P_{k}} \mathbf{A}_{\pi^{t},P_{k}} \right.\right.\nonumber\\
 & \left.\left. -\mathbf{\Pi}_{\pi_{n}^{t+1}}^{T} \mathbf{D}_{\pi^{t},\mu_{n},P_{n}} \mathbf{A}_{\pi^{t},P_{n}} \right)\right|\\
 = & \left| \mbox{Tr}\left(\mathbf{B}_{\pi^{t},\mu_{n},P_{n}} \mathbf{\Pi}_{\pi_{n}^{t+1}}^{T} \mathbf{D}_{\pi^{t},\mu_{n},P_{n}}\right) \right|, \label{eq:60}
\end{alignat}
where the second equation follows from (\ref{B}).

The Cauchy-Schwarz inequality \cite{baumgartner2011inequality} states that, for two complex matrices $\mathbf{X}$ and $\mathbf{Y}$, 
\begin{alignat}{1}
\left| \mbox{Tr} \left(\mathbf{X}^{\dagger}\mathbf{Y}\right)\right|  \le \sqrt{ \mbox{Tr}\left( \left| \mathbf{X} \right|\cdot \left| \mathbf{Y} \right| \right) \mbox{Tr} \left( \left| \mathbf{X}^{\dagger} \right|\cdot\left| \mathbf{Y}^{\dagger} \right|\right) }. \label{eq:matrixholdersine}
\end{alignat}
Accordingly, we can obtain  
\begin{alignat}{1}
 & \left| \mbox{Tr}\left(\mathbf{B}_{\pi^{t},\mu_{n},P_{n}} \mathbf{\Pi}_{\pi_{n}^{t+1}}^{T} \mathbf{D}_{\pi^{t},\mu_{n},P_{n}}  \right)\right|\nonumber\\
 \le & \sqrt{\mbox{Tr} \left(\left| \mathbf{B}_{\pi^{t},\mu_{n},P_{n}}^{T}\right| \left|\mathbf{\Pi}_{\pi_{n}^{t+1}}^{T}\mathbf{D}_{\pi^{t},\mu_{n},P_{n}}\right|\right)}\nonumber\\
 & \cdot\sqrt{\mbox{Tr}\left(\left| \mathbf{B}_{\pi^{t},\mu_{n},P_{n}}\right| \left|\mathbf{D}_{\pi^{t},\mu_{n},P_{n}}\mathbf{\Pi}_{\pi_{n}^{t+1}}\right|\right)}.\label{eq:61}
\end{alignat}


According to the definition in (\ref{absolute}), we can always find real positive semi-definite matrices $\left|\mathbf{B}_{\pi^{t},\mu_{n},P_{n}}^{T}\right|$,
$\left|\mathbf{B}_{\pi^{t},\mu_{n},P_{n}}\right|$, $\left|\mathbf{\Pi}_{\pi_{n}^{t+1}}^{T}\right|$
and $\left|\mathbf{\Pi}_{\pi_{n}^{t+1}}\right|$. By the Cauchy-Schwarz inequality, we can further obtain 
\begin{alignat}{1}
 & \mbox{Tr} \left(\left| \mathbf{B}_{\pi^{t},\mu_{n},P_{n}}^{T}\right| \left|\mathbf{\Pi}_{\pi_{n}^{t+1}}^{T}\mathbf{D}_{\pi^{t},\mu_{n},P_{n}}\right|\right) \nonumber\\
 \le & \sqrt{\mbox{Tr}\left(\left|\mathbf{B}_{\pi^{t},\mu_{n},P_{n}}^{T}\right|^{2}\right)\mbox{Tr}\left(\left|\mathbf{\Pi}_{\pi_{n}^{t+1}}^{T}\mathbf{D}_{\pi^{t},\mu_{n},P_{n}}\right|^{2}\right)},\label{eq:62}
\end{alignat}
and
\begin{alignat}{1}
 & \mbox{Tr}\left(\left|\mathbf{B}_{\pi^{t},\mu_{n},P_{n}}\right|\left|\mathbf{D}_{\pi^{t},\mu_{n},P_{n}}\mathbf{\Pi}_{\pi_{n}^{t+1}}\right|\right)\nonumber\\
 \le & \sqrt{\mbox{Tr}\left(\left|\mathbf{B}_{\pi^{t},\mu_{n},P_{n}}\right|^{2}\right)\mbox{Tr}\left(\left|\mathbf{D}_{\pi^{t},\mu_{n},P_{n}}\mathbf{\Pi}_{\pi_{n}^{t+1}}\right|^{2}\right)}.\label{eq:63}
\end{alignat}
Substituting (\ref{eq:62}) and (\ref{eq:63}) into (\ref{eq:61}) gives
\begin{alignat}{1}
G(\pi_{n}^{t+1}) & \le \left( \mbox{Tr}\left(\left|\mathbf{B}_{\pi^{t},\mu_{n},P_{n}}^{T}\right|^{2}\right)\mbox{Tr}\left(\left|\mathbf{\Pi}_{\pi_{n}^{t+1}}^{T}\mathbf{D}_{\pi^{t},\mu_{n},P_{n}}\right|^{2}\right)\right.\nonumber\\
 & \quad\left.\cdot\mbox{Tr}\left(\left|\mathbf{B}_{\pi^{t},\mu_{n},P_{n}}\right|^{2}\right)\mbox{Tr}\left(\left|\mathbf{D}_{\pi^{t},\mu_{n},P_{n}}\mathbf{\Pi}_{\pi_{n}^{t+1}}\right|^{2}\right)\right)^{\frac{1}{4}}.
\end{alignat}
Given
\begin{alignat}{1}
 & \mbox{Tr}\left(\left|\mathbf{\Pi}_{\pi_{n}^{t+1}}^{T}\mathbf{D}_{\pi^{t},\mu_{n},P_{n}} \right|^{2}\right)tr(\left|\mathbf{D}_{\pi^{t},\mu_{n},P_{n}}\mathbf{\Pi}_{\pi_{n}^{t+1}}\right|^{2})\nonumber\\
 = & \mbox{Tr}\left(\mathbf{D}_{\pi^{t},\mu_{n},P_{n}} \mathbf{\Pi}_{\pi_{n}^{t+1}} \mathbf{\Pi}_{\pi_{n}^{t+1}}^{T} \mathbf{D}_{\pi^{t},\mu_{n},P_{n}}\right)\nonumber\\
 & \cdot\mbox{Tr}\left(\mathbf{\Pi}_{\pi_{n}^{t+1}}^{T} \mathbf{D}_{\pi^{t},\mu_{n},P_{n}} \mathbf{D}_{\pi^{t},\mu_{n},P_{n}} \mathbf{\Pi}_{\pi_{n}^{t+1}}\right)\\
 = & \mbox{Tr}\left(\mathbf{D}_{\pi^{t},\mu_{n},P_{n}} \mathbf{\Pi}_{\pi_{n}^{t+1}} \mathbf{\Pi}_{\pi_{n}^{t+1}}^{T} \mathbf{D}_{\pi^{t},\mu_{n},P_{n}}\right)^{2},
\end{alignat}
and similarly 
\begin{alignat}{1}
 & \mbox{Tr}\left(\left|\mathbf{B}_{\pi^{t},\mu_{n},P_{n}}^{T}\right|^{2}\right)\mbox{Tr} \left(\left|\mathbf{B}_{\pi^{t},\mu_{n},P_{n}}\right|^{2}\right)\nonumber\\ 
 = & \mbox{Tr}\left(\mathbf{B}_{\pi^{t},\mu_{n},P_{n}} \mathbf{B}_{\pi^{t},\mu_{n},P_{n}}^{T}\right)^{2},
\end{alignat}
we have
\begin{alignat}{1}
\label{gap}
G(\pi_{n}^{t+1}) & \le \sqrt{\mbox{Tr}\left(\mathbf{B}_{\pi^{t},\mu_{n},P_{n}} \mathbf{B}_{\pi^{t},\mu_{n},P_{n}}^{T}\right)}\nonumber\\
& \quad\cdot\sqrt{\mbox{Tr}\left(\mathbf{D}_{\pi^{t},\mu_{n},P_{n}} \mathbf{\Pi}_{\pi_{n}^{t+1}} \mathbf{\Pi}_{\pi_{n}^{t+1}}^{T} \mathbf{D}_{\pi^{t},\mu_{n},P_{n}}\right)}\\ 
 & =\left\Vert \mathbf{B}_{\pi^{t},\mu_{n},P_{n}}\right\Vert_{F}\sqrt{\sum_{s,a}[\rho_{\pi^{t},\mu_{n},P_{n}}(s)\pi_{n}^{t+1}(a\vert s)]^{2}}.
\end{alignat}
Thus, we can further obtain \footnote{Note that we have used the notation $\left(\pi^{\prime}-\pi\right)(s,a)=\pi^{\prime}(a\vert s)-\pi(a\vert s),\forall s\in \mathcal{S},a\in \mathcal{A}$.}
\begin{alignat}{1}
\label{eq:82}
& G(\pi_{n}^{t+1}-\pi_{n}^{t})\\
\le & \left\Vert \mathbf{B}_{\pi^{t},\mu_{n},P_{n}}\right\Vert _{F}\sqrt{\sum_{s,a}\left[\rho_{\pi^{t},\mu_{n},P_{n}}(s) \left(\pi_{n}^{t+1}(a\vert s)-\pi^{t}(a\vert s)\right)\right]^{2}}.\nonumber
\end{alignat}

Given $\sum_{a}\pi^{t}(a\vert s)A_{\pi^{t},\mu}(s,a)=0,\forall{s\in\mathcal{S}}$, we have 
\begin{alignat}{1}
& \mathbb{A}_{\pi^{t},\mu_{n},P_{n}}(\pi_{n}^{t+1})\nonumber\\
= & \sum_{s}\rho_{\pi^{t},\mu_{n},P_{n}}(s)\sum_{a}\pi_{n}^{t+1}(a|s)A_{\pi^{t},P_{n}}(s,a)\nonumber\\
& \quad-\sum_{s}\rho_{\pi^{t},\mu_{n},P_{n}}(s)\sum_{a}\pi^{t}(a|s)A_{\pi^{t},P_{n}}(s,a) \\
= & \sum_{s}\rho_{\pi^{t},\mu_{n},P_{n}}(s)\sum_{a}\left[\pi_{n}^{t+1}(a|s)-\pi^{t}(a\vert s)\right]A_{\pi^{t},P_{n}}(s,a)\\
= & \mathbb{A}_{\pi^{t},P_{n}}\left(\left(\pi_{n}^{t+1}-\pi^{t}\right)\right).
\end{alignat}
It thus follows from (\ref{G}) that  
\begin{eqnarray}
\label{gapgap}
G(\pi_{n}^{t+1})  = G(\pi_{n}^{t+1}-\pi_{n}^{t}).  
\end{eqnarray}
By combining (\ref{gapgap}) and (\ref{eq:82}), we can obtain 
\begin{alignat}{1}
G(\pi_{n}^{t+1}) & \le \left\Vert \mathbf{B}_{\pi^{t},\mu_{n},P_{n}}\right\Vert _{F}\\
& \quad \cdot \sqrt{\sum_{s,a}\left[\rho_{\pi^{t},\mu_{n},P_{n}}(s) \left(\pi_{n}^{t+1}(a\vert s)-\pi^{t}(a\vert s)\right)\right]^{2}}.\nonumber
\end{alignat}
Based on (\ref{G}), we can obtain
\begin{alignat}{1}
& \sum_{k=1}^{N}q_{k}\mathbb{A}_{\pi^{t},\mu_{k},P_{k}}(\pi_{n}^{t+1})\nonumber\\
\ge & \mathbb{A}_{\pi^{t},\mu_{n},P_{n}}(\pi_{n}^{t+1})-\left\Vert \mathbf{B}_{\pi^{t},\mu_{n},P_{n}}\right\Vert _{F}\nonumber\\
& \cdot\sqrt{\sum_{s,a}\left[\rho_{\pi^{t},\mu_{n},P_{n}}(s) \left(\pi_{n}^{t+1}(a\vert s)-\pi^{t}(a\vert s)\right)\right]^{2}}.
\end{alignat}
Given $\sqrt{\left\vert a \right\vert^{2} + \left\vert b \right\vert^{2}} \le \sqrt{\left(\left\vert a \right\vert + \left\vert b \right\vert\right)^{2}}=\left\vert a \right\vert + \left\vert b \right\vert$, we can obtain
\begin{alignat}{1}
& \sum_{k=1}^{N}q_{k}\mathbb{A}_{\pi^{t},\mu_{k},P_{k}}(\pi_{n}^{t+1})\nonumber\\
\ge & \mathbb{A}_{\pi^{t},\mu_{n},P_{n}}(\pi_{n}^{t+1})-\left\Vert \mathbf{B}_{\pi^{t},\mu_{n},P_{n}}\right\Vert _{F}\nonumber\\
& \cdot\sum_{s,a}\rho_{\pi^{t},\mu_{n},P_{n}}(s)\left|\pi_{n}^{t+1}(a\vert s)-\pi^{t}(a\vert s)\right|.
\end{alignat}
Denotes the TV distance between $\pi(\cdot\vert s)$ and $\pi^{\prime}(\cdot\vert s)$ as
$D_{TV}(\pi(\cdot\vert s)\Vert\pi^{\prime}(\cdot\vert s))  =\frac{1}{2}\sum_{a}\left|\pi(a\vert s)-\pi^{\prime}(a\vert s)\right|$. We can finally obtain 
\begin{alignat}{1}
\label{eq:86} 
& \sum_{k=1}^{N}q_{k}\mathbb{A}_{\pi^{t},\mu_{k},P_{k}}(\pi_{n}^{t+1})\\
\ge & \mathbb{A}_{\pi^{t},\mu_{n},P_{n}}(\pi_{n}^{t+1}) - \alpha \mathbb{E}_{s\thicksim\rho_{\pi^{t},\mu_{n},P_{n}}} \left[D_{TV}(\pi(\cdot\vert s)\Vert\pi^{\prime}(\cdot\vert s))\right]\nonumber,
\end{alignat}
where $\alpha = 2\left\Vert \mathbf{B}_{\pi^{t},\mu_{n},P_{n}}\right\Vert _{F}$ 
\end{proof}

\section{Proof of Corollary \ref{cor} \label{corproof}}
\begin{proof}
It follows from (\ref{eq:vectorpolicyadvantage}) that $\mathbb{A}_{\pi^t,\mu_{n},P_{n}}(\pi_n^{t+1}) = \mbox{Tr} \left(\mathbf{D}_{\pi^t,\mu_{n},P_{n}} \mathbf{A}_{\pi^t,P_{n}}\mathbf{\Pi}_{\pi_n^{t+1}}^{T} \right).$ 
By following the same procedure from (\ref{eq:60}) to (\ref{eq:82}), we can prove  
\begin{alignat}{1}
& \mathbb{A}_{\pi^{t},\mu_{n},P_{n}}(\pi_{n}^{t+1})\nonumber\\
\le & \left\Vert \mathbf{A}_{\pi^{t},P_{n}}\right\Vert _{F}\nonumber\\
& \cdot\sqrt{\sum_{s,a}\left[\rho_{\pi^{t},\mu_{n},P_{n}}(s) \left(\pi_{n}^{t+1}(a\vert s)-\pi^{t}(a\vert s)\right)\right]^{2}}.\label{eq:90}
\end{alignat}
Thus, if $\left\Vert \mathbf{A}_{\pi^{t},P_{n}}\right\Vert _{F}\le\left\Vert \mathbf{B}_{\pi^{t},\mu_{n},P_{n}}\right\Vert _{F}$, then 
\begin{alignat}{1}
& \mathbb{A}_{\pi^{t},\mu_{n},P_{n}}(\pi_{n}^{t+1})\nonumber\\ 
 \le & \left\Vert \mathbf{B}_{\pi^{t},\mu_{n},P_{n}}\right\Vert _{F}\nonumber\\
 & \cdot\sqrt{\sum_{s,a}\left[\rho_{\pi^{t},\mu_{n},P_{n}}(s) \left(\pi_{n}^{t+1}(a\vert s)-\pi^{t}(a\vert s)\right)\right]^{2}}.
\end{alignat}
As a result, the right-hand side of (\ref{eq:84}) will be less or equal to zero, where the equality holds when $\pi_{n}^{t+1}=\pi^t$. This indicates that there is no benefit from further updating $\pi^t$ at the $n$-th agent, and there will be no improvement for the global objective. Thus, (\ref{BleD1}) is a necessary condition for the local update to be able to improve the global objective.
\end{proof}

\section{Proof of Corollary \ref{cor:1} \label{proof:corollary1}}
\begin{proof}
It follows from (\ref{eq:2}) that the global discounted reward at time $t+1$ can be given by 
\begin{eqnarray}
\label{ACE1}
\text{\ensuremath{\eta}}(\pi_n^{t+1}) =\sum_{k=1}^{N}q_{k}\eta_{k}(\pi_n^{t+1}). 
\end{eqnarray}
By substituting (\ref{eq:cporeturndiff}) into (\ref{ACE1}), we can obtain
\begin{alignat}{1}
\label{ACE2}
\text{\ensuremath{\eta}}(\pi_n^{t+1}) & \ge \sum_{k=1}^{N} q_{k} \left( L_{\pi^t,\mu_{k},P_{k}}(\pi_n^{t+1})\right.\\
& \quad\left.-\text{\ensuremath{c_{k}^{\text{CPO}} \mathbb{E}_{s\thicksim\rho_{\pi^{t},\mu_{k},P_{k}}} \left[D_{TV}(\pi^{t}(\cdot\vert s)\Vert\pi_{n}^{t+1}(\cdot\vert s))\right]}} \right).\nonumber
\end{alignat}
By substituting (\ref{eq:surrogate}) into (\ref{ACE2}) and bounding $c_{k}^{\text{CPO}}$ from below by $c_{k}^{\prime}$, we have
\begin{alignat}{1}
\label{eq:66}
\text{\ensuremath{\eta}}(\pi_n^{t+1}) & \ge \sum_{k=1}^{N}q_{k} \left( 
\eta_{k}(\pi^{t})+ \mathbb{A}_{\pi^{t},P_{k}}(\pi_n^{t+1})\right.\\
& \quad\left.-\text{\ensuremath{c'_k \mathbb{E}_{s\thicksim\rho_{\pi^{t},\mu_{k},P_{k}}} \left[D_{TV}(\pi^{t}(\cdot\vert s)\Vert\pi_{n}^{t+1}(\cdot\vert s))\right]}} \right),\nonumber
\end{alignat}
where $c'_{k}=\frac{4\epsilon_{k}\gamma}{(1-\gamma)^{2}}$, and $\epsilon_{k}=\max_{s,a}|A_{\pi^{t},P_{k}}(s,a)|$ denotes the maximum absolute value of the advantage function over all state and action pairs.

To remove the dependency on $\rho_{\pi^{t},\mu_{k},P_{k}}$, we have
\begin{alignat}{1}
\label{eq:54}
& \mathbb{E}_{s\thicksim\rho_{\pi^{t},\mu_{k},P_{k}}} \left[D_{TV}(\pi^{t}(\cdot\vert s)\Vert\pi_{n}^{t+1}(\cdot\vert s))\right]\nonumber\\
& - \mathbb{E}_{s\thicksim\rho_{\pi^{t},\mu_{n},P_{n}}} \left[D_{TV}(\pi^{t}(\cdot\vert s)\Vert\pi_{n}^{t+1}(\cdot\vert s))\right]\nonumber\\
= & \sum_{s}\rho_{\pi^{t},\mu_{k},P_{k}}(s)D_{TV}(\pi^{t}(\cdot\vert s)\Vert\pi_{n}^{t+1}(\cdot\vert s))\nonumber\\
& - \sum_{s}\rho_{\pi^{t},\mu_{n},P_{n}}(s)D_{TV}(\pi^{t}(\cdot\vert s)\Vert\pi_{n}^{t+1}(\cdot\vert s))\\
\le & D_{TV}^{\max}(\pi^{t},\pi_{n}^{t+1})\sum_{s} \left\vert\rho_{\pi^{t},\mu_{k},P_{k}}(s)-\rho_{\pi^{t},\mu_{n},P_{n}}(s)\right\vert,
\end{alignat}
where $D_{TV}^{max}(\pi,\pi^{\prime})  =\max_{s} D_{TV} \left( \pi(\cdot\vert s) \Vert \pi^{\prime}(\cdot\vert s) \right)$ denotes the maximum TV divergence between two policies $\pi$ and $\pi^{\prime}$
among all states. It thus follows from (\ref{eq:54}) that
\begin{alignat}{1}
\label{eq:55}
& \mathbb{E}_{s\thicksim\rho_{\pi^{t},\mu_{k},P_{k}}} \left[D_{TV}(\pi^{t}(\cdot\vert s)\Vert\pi_{n}^{t+1}(\cdot\vert s))\right]\nonumber\\
\le & \mathbb{E}_{s\thicksim\rho_{\pi^{t},\mu_{n},P_{n}}} \left[D_{TV}(\pi^{t}(\cdot\vert s)\Vert\pi_{n}^{t+1}(\cdot\vert s))\right]\nonumber\\
& + 2 D_{TV}^{\max}(\pi^{t},\pi_{n}^{t+1}) D_{TV}(\rho_{\pi^{t},\mu_{k},P_{k}} \Vert \rho_{\pi^{t},\mu_{n},P_{n}}).
\end{alignat}

Substituting (\ref{eq:thm1}) and (\ref{eq:55}) into (\ref{eq:66}) gives:
\begin{alignat}{1}
\eta(\pi_{n}^{t+1}) & \ge \sum_{k=1}^{N}q_{k}\eta_{k}(\pi^{t})+\mathbb{A}_{\pi^{t},\mu_{n},P_{n}}(\pi_{n}^{t+1})\nonumber\\
& \quad-(\alpha+\beta) \mathbb{E}_{s\thicksim\rho_{\pi^{t},\mu_{n},P_{n}}} \left[D_{TV}(\pi^{t}(\cdot\vert s)\Vert\pi_{n}^{t+1}(\cdot\vert s))\right]\nonumber\\
& \quad-\delta D_{TV}^{\max}(\pi^{t},\pi_{n}^{t+1}),
\end{alignat}
where $\delta=2 \sum_{k=1}^{N}q_{k} \frac{4\epsilon_{k}\gamma}{(1-\gamma)^{2}} D_{TV}(\rho_{\pi^{t},\mu_{k},P_{k}} \Vert \rho_{\pi^{t},\mu_{n},P_{n}})$, $\beta=\sum_{k=1}^{N}q_{k} \frac{4\epsilon_{k}\gamma}{(1-\gamma)^{2}}$ and $\epsilon_{k}=\max_{s,a}\left|A_{\pi^{t},P_{k}}(s,a)\right|$.
\end{proof}

\section{Effectiveness of Improving the Local Objective\label{proof:MMAlgo}}

Let $g(\pi_{n}^{t+1})$ denote the RHS of (\ref{eq:cor1}) as a function of $\pi_{n}^{t+1}$. Thus, we have the following relation 
\begin{alignat}{1}
\eta(\pi_{n}^{t+1}) & \ge g(\pi_{n}^{t+1}),\label{eq:mm1}\\
\eta(\pi^{t}) & = g(\pi^{t}),\label{eq:mm2}\\
\eta(\pi_{n}^{t+1})-\eta(\pi^{t}) & \ge g(\pi_{n}^{t+1})-g(\pi^{t}).\label{eq:mm}
\end{alignat}
The first inequality comes from (\ref{eq:cor1}). The equality
in (\ref{eq:mm2}) can be checked by replacing $\pi_{n}^{t+1}$ with $\pi^{t}$ in the RHS of (\ref{eq:cor1}). 
Finally, subtracting (\ref{eq:mm2}) from (\ref{eq:mm1}) gives (\ref{eq:mm}).
The above conditions guarantee that we can improve the reward $\eta(\pi_{n}^t)$ at each iteration $t$ by improving the surrogate term $g(\pi_{n}^t)$. In particular, the surrogate function $g(\pi)$ minorizes $\eta(\pi)$ at $\pi^{t}$ \cite{DBLP:journals/corr/SchulmanLMJA15,hunter2004tutorial}.
\section{Proof of Theorem \ref{thm:monotonicity}\label{proof:monotonicity}}
\begin{lem}
\label{lem:linearity}
(Linearity of policy advantage $\mathbb{A}_{\pi}$ and surrogate objective $L_{\pi}$). Given arbitrary policies $\pi$, $\hat{\pi}$ and $\pi^{\prime}$, the policy advantage of linearly combinated policy $\gamma\hat{\pi}+(1-\gamma)\pi^{\prime}$ over $\pi$ is the linear combination of individual policy advantages, i.e. for all $\gamma \le 1$, we have
\begin{alignat}{1}
\mathbb{A}_{\pi}(\gamma\hat{\pi}+(1-\gamma)\pi^{\prime}) & =\gamma \mathbb{A}_{\pi}(\hat{\pi})+(1-\gamma)\mathbb{A}_{\pi}(\pi^{\prime})\\
L_{\pi}(\gamma\hat{\pi}+(1-\gamma)\pi^{\prime}) & =\gamma L_{\pi}(\hat{\pi})+(1-\gamma)L_{\pi}(\pi^{\prime}).
\end{alignat}
\end{lem}
\begin{proof}
\begin{alignat}{1}
& L_{\pi}(\gamma\hat{\pi}+(1-\gamma)\pi^{\prime})\nonumber\\
= & \eta(\pi)+\gamma\sum_{s}\rho_{\pi}(s)\sum_{a}\hat{\pi}(a|s)A_{\pi}(s,a)\nonumber\\
& +(1-\gamma)\sum_{s}\rho_{\pi}(s)\sum_{a}\pi^{\prime}(a|s)A_{\pi}(s,a)\\
= & \eta(\pi) + \gamma \mathbb{A}_{\pi}(\hat{\pi})+(1-\gamma)\mathbb{A}_{\pi}(\pi^{\prime})\\
= & \gamma L_{\pi}(\hat{\pi})+(1-\gamma)L_{\pi}(\pi^{\prime}).
\end{alignat}
\end{proof}
\begin{lem}
\label{lem:tvupper}
In tabular case, for all states, the TV divergence between $\pi^{t+1}(\cdot \vert s)$ and $\pi^{t}(\cdot \vert s)$ is upper bounded by the expectation of the TV divergence between local policies $\pi_{k}^{t+1}(\cdot \vert s)$ and $\pi^{t}(\cdot \vert s)$.
\end{lem}
\begin{proof}
For all states $s$,
\begin{alignat}{1}
& D_{TV}(\pi^{t}(\cdot \vert s),\pi^{t+1})(\cdot \vert s)]\nonumber\\
= & \frac{1}{2}\left\Vert \pi^{t}(\cdot \vert s)-\pi^{t+1}(\cdot \vert s)\right\Vert _{1}\\
= & \frac{1}{2}\left\Vert \sum_{n=1}^{N}p_{k}(\pi^{t}(\cdot \vert s)-\pi_{k}^{t+1}(\cdot \vert s))\right\Vert _{1}\\
\le & \sum_{k=1}^{N}q_{k}\frac{1}{2}\left\Vert \pi^{t}(\cdot \vert s)-\pi_{k}^{t+1}(\cdot \vert s)\right\Vert _{1}\\
= & \sum_{k=1}^{N}q_{k}\left[ D_{TV}(\pi^{t}(\cdot \vert s)\lVert\pi_{k}^{t+1}(\cdot \vert s) \right]\\
= & \mathbb{E}_{k}\left[ D_{TV}(\pi^{t}(\cdot \vert s)\lVert\pi_{k}^{t+1}(\cdot \vert s) \right].
\end{alignat}
\end{proof}

Now the proof of Theorem \ref{thm:monotonicity} follows.
\begin{proof}
By (\ref{eq:cporeturndiff}) and bounding $c_{n}^{\text{CPO}}$ by $c_{n}$, we can obtain
\begin{alignat}{1}
\eta(\pi^{t+1}) & =\sum_{n=1}^{N}q_{n}\eta_{n}(\sum_{k=1}^{N}q_{k}\pi_{k}^{t+1})\\
& \ge \sum_{n=1}^{N}p_{n}\left\{ \eta_{n}(\pi^{t})+\mathbb{A}_{\pi^{t},\mu_{n},P_{n}}\left(\sum_{k=1}^{N}q_{k}\pi_{k}^{t+1}\right)\right.\nonumber\\
& \quad\left.-c_{n} \mathbb{E}_{s\thicksim\rho_{\pi,\mu_{n},P_{n}}} \left[D_{TV}(\pi^{t}(\cdot\vert s)\Vert\pi^{t+1}(\cdot\vert s))\right]\right\}.
\end{alignat}
By the linearity of advantage function, i.e. Lemma \ref{lem:linearity}
, we have
\begin{alignat}{1}
\eta(\pi^{t+1}) & \ge \eta(\pi^{t})+\sum_{n=1}^{N}q_{n}\left\{ \mathbb{A}_{\pi^{t},\mu_{n},P_{n}}(\sum_{k=1}^{N}q_{k}\pi_{k}^{t+1})\right.\nonumber\\
& \quad\left.-c_{n} \mathbb{E}_{s\thicksim\rho_{\pi,\mu_{n},P_{n}}} \left[D_{TV}(\pi^{t}(\cdot\vert s)\Vert\pi^{t+1}(\cdot\vert s))\right]\right\}\\
& = \eta(\pi^{t})+\sum_{n=1}^{N}q_{n}\left\{ \sum_{k=1}^{N}q_{k} \mathbb{A}_{\pi^{t},\mu_{n},P_{n}} (\pi_{k}^{t+1})\right.\nonumber\\
& \quad\left.-c_{n} \mathbb{E}_{s\thicksim\rho_{\pi,\mu_{n},P_{n}}} \left[D_{TV}(\pi^{t}(\cdot\vert s)\Vert\pi^{t+1}(\cdot\vert s))\right]\right\}.
\end{alignat}
By Lemma \ref{lem:tvupper}, we have
\begin{alignat}{1}
& \eta(\pi^{t+1})\nonumber\\
\ge & \eta(\pi^{t})+\sum_{n=1}^{N}q_{n}\left\{ \sum_{k=1}^{N}q_{k} \mathbb{A}_{\pi^{t},\mu_{n},P_{n}}(\pi_{k}^{t+1})\right.\nonumber\\
& \left.-c_{n} \mathbb{E}_{s\thicksim\rho_{\pi,\mu_{n},P_{n}}} \left[ \sum_{k}^{N}q_{k}D_{TV}(\pi^{t}(\cdot\vert s)\Vert\pi_{k}^{t+1}(\cdot\vert s))\right ]\right\}\\
= & \eta(\pi^{t})+\sum_{n=1}^{N}q_{n}\left\{ \sum_{k=1}^{N}q_{k} \mathbb{A}_{\pi^{t},\mu_{n},P_{n}} (\pi_{k}^{t+1})\right.\nonumber\\
& \left.-\sum_{k=1}^{N}q_{k} c_{n} \mathbb{E}_{s\thicksim\rho_{\pi,\mu_{n},P_{n}}} \left[ D_{TV}(\pi^{t}(\cdot\vert s)\Vert\pi_{k}^{t+1}(\cdot\vert s))\right ]\right\}.
\end{alignat}
By exchanging the summation order, we have
\begin{alignat}{1}
\eta(\pi^{t+1}) & \ge \eta(\pi^{t})+\sum_{k=1}^{N}q_{k}\left\{ \sum_{n=1}^{N}q_{n} \mathbb{A}_{\pi^{t},\mu_{n},P_{n}} (\pi_{k}^{t+1})\right.\nonumber\\
& \quad\left.-\beta \mathbb{E}_{s\thicksim\rho_{\pi,\mu_{n},P_{n}}} \left[ D_{TV}(\pi^{t}(\cdot\vert s)\Vert\pi_{k}^{t+1}(\cdot\vert s))\right ]\right\}.
\end{alignat}
By following the same procedure from Appendix \ref{proof:corollary1}, we have
\begin{alignat}{1}
\eta(\pi^{t+1}) & \ge \eta(\pi^{t})+\sum_{k=1}^{N}q_{k}\left\{ \mathbb{A}_{\pi^{t},\mu_{k},P_{k}}(\pi_{k}^{t+1})\right.\nonumber\\
& \quad\left.-(\alpha+\beta) \mathbb{E}_{s\thicksim\rho_{\pi,\mu_{k},P_{k}}} \left[ D_{TV}(\pi^{t}(\cdot\vert s)\Vert\pi_{k}^{t+1}(\cdot\vert s))\right ]\right\}\nonumber\\
& \quad\left.-\delta D_{TV}^{\max}(\pi^{t},\pi_{k}^{t+1})\right\}.
\label{eq:76}
\end{alignat}
This proves that the RHS of Corollary \ref{cor:1} forms a lower bound on the policy performance difference between $\pi^{t+1}$ and $\pi^{t}$. Moreover, since the equality holds when $\pi_{k}^{t+1}=\pi^{t},\forall k=0,...,N$, optimizing the RHS of (\ref{eq:76}) or (\ref{eq:prelocalobj}) in each round can monotonically improve the policy performance difference between $\pi^{t+1}$ and $\pi^{t}$ due to the proof in Appendix \ref{proof:MMAlgo}. This completes the proof of Theorem \ref{thm:monotonicity}.
\end{proof}
\section{Proof of Theorem \ref{thm:2}\label{proof:theorem2}}

This proof uses some useful lemmas and corollaries from \cite{kuba2022trust}.
\begin{lem}
\label{lem:kuba4}
(Continuity of $\rho_{\pi}$, \cite[Lemma 4]{kuba2022trust}). The (unnormalized) discounted visitation frequency $\rho_{\pi}$ is continuous in $\pi$.
\end{lem}
\begin{lem}
(Continuity of $Q_{\pi}$, \cite[Lemma 5]{kuba2022trust}). Let $\pi$ be a policy. Then $Q_{\pi}\left(s,a\right)$
is Lipschitz-continuous in $\pi$.
\end{lem}
\begin{cor}
\label{cor:kuba}
(\cite[Corollary 1]{kuba2022trust}). Due to the continuity of $Q_{\pi}$, the following functions are Lipschitz-continuous
in $\pi$: The state value function $V_{\pi}$, the advantage function
$A_{\pi}\left(s,a\right)$, and the expected total reward $\eta\left(\pi\right)$.
\end{cor}
\begin{lem}
\label{lem:kuba6}
(Continuity of policy advantage, \cite[Lemma 6]{kuba2022trust}). Let $\pi$ and $\hat{\pi}$ be policies. The policy advantage $\mathbb{A}_\pi(\hat{\pi})$ is continuous in $\pi$.
\end{lem}
\begin{lem}
\label{lem:continuityofb}
(Continuity of $\left\Vert \mathbf{B}_{\pi,\mu_{n},P_{n}} \right\Vert _{F}$). For every agent $k=1,...,N$, the level of heterogeneity is continuous in $\pi$.
\end{lem}
\begin{proof}
To simplify the proof, we expand the matrix norm $\left\Vert \mathbf{B}_{\pi,\mu_{n},P_{n}} \right\Vert _{F}$ as follows
\begin{alignat}{1}
& \left\Vert \mathbf{B}_{\pi,\mu_{n},P_{n}} \right\Vert _{F}\nonumber\\
= & \left\Vert \sum_{k=1}^{N} q_{k} \mathbf{D}_{\pi,\mu_{n},P_{n}}^{-1} \mathbf{D}_{\pi,\mu_{k},P_{k}} \mathbf{A}_{\pi,P_{k}} - \mathbf{A}_{\pi,P_{n}} \right\Vert _{F}\\
= & \sum_{k=1}^{N} q_{k} \sum_{s}\sum_{a} \left[ \frac{\rho_{\pi,\mu_{k},P_{k}}(s)}{\rho_{\pi,\mu_{n},P_{n}}(s)} A_{\pi,P_{k}}(s,a) - A_{\pi,P_{n}}(s,a) \right]^{2}.
\end{alignat}
The continuity of $\left\Vert \mathbf{B}_{\pi,\mu_{n},P_{n}} \right\Vert _{F}$ follows from the assumption that every state is reachable and the continuity of $\rho_{\pi}(s)$ and $A_{\pi}(s,a)$.
\end{proof}
\begin{lem}
\label{lem:continuityofh}
(Continuity of local objective). For every agent $k=1,...,N$, the local objective $h_{k}(\pi^{\prime};\pi)$ in (\ref{eq:prelocalobj}) is continuous in $\pi$.
\end{lem}
\begin{proof}
By Lemma \ref{lem:kuba6}, $A_{\hat{\pi},\mu_{k},P_{k}}(\pi)$ is continuous. By Corollary \ref{cor:kuba} and Lemma \ref{lem:continuityofb}, $\alpha$, $\beta$ and $\delta$ are continuous. By the continuity of $D_{TV}$ and $\max$ operator, the last two terms of (\ref{eq:prelocalobj}) are continuous. As a result, $h_{k}(\pi^{\prime};\pi)$ is continuous in $\pi$.
\end{proof}
\begin{defn}
\label{def:stationarity}
(FedKL-Stationarity) A policy $\pi^{\prime}$ is FedKL-stationary
if, for every agent $k=1,...,N$,
\begin{alignat}{1}
\pi^{\prime} & =\arg\max_{\pi}\left[\mathbb{A}_{\pi^{\prime},\mu_{k},P_{k}}(\pi)\right.\nonumber\\
& \quad-\left.\left(\alpha+\beta\right)\mathbb{E}_{s\thicksim\rho_{\pi^{\prime},\mu_{k},P_{k}}}\left[D_{TV}(\pi^{\prime}(\cdot \vert s),\pi(\cdot \vert s))\right]\right.\nonumber\\
& \quad\left.-\delta D_{TV}^{\max}(\pi^{\prime},\pi)\right]\\
& = \arg\max_{\pi}\left[h_{k}(\pi;\pi^{\prime})\right].
\end{alignat}
\end{defn}

Next, we prove Theorem \ref{thm:2}. To this end, we start from the convergence of (\ref{eq:prelocalobj}) and then show the limit points of the sequence generated by Algorithm \ref{alg:PAvg} are FedKL-stationary.
\begin{proof}
By Theorem \ref{thm:monotonicity}, the sequence
$\left(\eta\left(\pi^{t}\right)\right)_{t=0}^{\infty}$ is non-decreasing and bounded above by $\frac{R_{\max}}{1-\gamma}$, where $R_{\max}$ is the maximum of reward. As a result, we know the sequence converges. Let $\hat{\eta}$ denote the limit point. Given the sequence
of policies $\left(\pi^{t}\right)_{t=0}^{\infty}$ is bounded, we know it has at least one convergent subsequence, by the Bolzano-Weierstrass Theorem.
Denote $\hat{\pi}$ as any limit point of the sequence, and let $\left(\pi^{t_{j}}\right)_{t_{j}=0}^{\infty}$
be a subsequence converging to $\hat{\pi}$. By the continuity of $\eta$
in $\pi$ (Corollary \ref{cor:kuba}), we have
\begin{alignat}{1}
\eta\left(\hat{\pi}\right) & =\eta\left(\lim_{j\to\infty}\pi^{t_{j}}\right)=\lim_{j\to\infty}\eta\left(\pi^{t_{j}}\right)=\hat{\eta}.
\end{alignat}

We will now establish the FedKL-stationarity of any limit point $\hat{\pi}$.
By Theorem \ref{thm:monotonicity}, we have
\begin{alignat}{1}
0 & =\lim_{t\to\infty}\left[\eta\left(\pi^{t+1}\right)-\eta\left(\pi^{t}\right)\right]\\
 & \ge\lim_{t\to\infty}\mathbb{E}_{k}\left[\mathbb{A}_{\pi^{t},\mu_{k},P_{k}}(\pi_{k}^{t+1})\right.\nonumber\\
 & \quad\left.-\left(\alpha+\beta\right)\mathbb{E}_{s\thicksim\rho_{\pi^{t},\mu_{k},P_{k}}}\left[D_{TV}(\pi^{t}\left(\cdot\vert s\right),\pi_{k}^{t+1}\left(\cdot\vert s\right))\right]\right.\nonumber\\
 & \quad\left.-\delta D_{TV}^{\max}(\pi^{t},\pi_{k}^{t+1})\right].
\end{alignat}
Consider an arbitrary limit point $\hat{\pi}$ and a subsequence $\left(\pi^{t_{j}}\right)_{j=0}^{\infty}$ that
converges to $\hat{\pi}$. Then, we can obtain
\begin{alignat}{1}
0 & \ge\lim_{j\to\infty}\mathbb{E}_{k}\left[\mathbb{A}_{\pi^{t_{j}},\mu_{k},P_{k}}(\pi_{k}^{t_{j}+1})\right.\nonumber\\
& \quad\left.-\left(\alpha+\beta\right)\mathbb{E}_{s\thicksim\rho_{\pi^{t_{j}},\mu_{k},P_{k}}}\left[D_{TV}(\pi^{t_{j}}\left(\cdot\vert s\right),\pi_{k}^{t_{j}+1}\left(\cdot\vert s\right))\right]\right.\nonumber\\
 & \quad\left.-\delta D_{TV}^{\max}(\pi^{t_{j}},\pi_{k}^{t_{j}+1})\right].
\label{eq:0gelimit}
\end{alignat}
As every agent is optimizing its local objective, the expectation is with respect to non-negative random variables. Thus, for arbitrary $k=1,...,N$, the RHS of (\ref{eq:0gelimit}) is bounded from below by
\begin{alignat}{1}
& q_{k} \lim_{j\to\infty} \max_{\pi} \left[\mathbb{A}_{\pi^{t_{j}},\mu_{k},P_{k}}(\pi)\right.\nonumber\\
& \left.-\left(\alpha+\beta\right)\mathbb{E}_{s\thicksim\rho_{\pi^{t_{j}},\mu_{k},P_{k}}}\left[D_{TV}(\pi^{t_{j}}\left(\cdot\vert s\right),\pi\left(\cdot\vert s\right))\right]\right]\nonumber\\
& \left.-\delta D_{TV}^{\max}(\pi^{t_{j}},\pi)\right] \ge 0.
\label{eq:lim84}
\end{alignat}
By Lemma \ref{lem:continuityofh} and the Squeeze theorem, it follows from (\ref{eq:0gelimit}) and (\ref{eq:lim84}) that
\begin{alignat}{1}
\lim_{j\to\infty} \max_{\pi}\left[h_{k}(\pi;\pi^{t_{j}})\right]=\max_{\pi}\left[h_{k}(\pi;\hat{\pi})\right]=0.
\end{alignat}
This proves that, for any limit point of the policy sequence induced
by Algorithm \ref{alg:PAvg}, $\max_{\pi}\left[h_{k}(\pi;\hat{\pi})\right]=0$ for every agent $k=1,...,N$,
which is equivalent to Definition \ref{def:stationarity}.

\end{proof}
}


\bibliography{references}
\bibliographystyle{IEEEtran}


 





\end{document}